\documentclass{article}

\PassOptionsToPackage{square,numbers}{natbib}
\PassOptionsToPackage{dvipsnames}{xcolor}

\usepackage[final]{neurips_2026}
\makeatletter
\providecommand{\@trackname}{}
\makeatother

\usepackage[utf8]{inputenc} 
\usepackage[T1]{fontenc}    
\usepackage{hyperref}       
\usepackage{url}            
\usepackage{booktabs}       
\usepackage{amsfonts}       
\usepackage{microtype}      
\usepackage{xcolor}         

\usepackage{tikz}
\usetikzlibrary{positioning, fit, backgrounds, arrows.meta}
\usetikzlibrary{calc} 
\usepackage[dvipsnames]{xcolor} 

\usepackage{enumitem}
\usepackage{amsmath}
\usepackage{amssymb}
\usepackage{amsthm}
\usepackage{mathtools}
\usepackage{graphicx}
\usepackage{algorithm}
\usepackage{algorithmicx}
\usepackage{algpseudocode}
\usepackage[capitalize,noabbrev]{cleveref}

\newtheorem{theorem}{Theorem}[section]
\newtheorem{proposition}[theorem]{Proposition}
\newtheorem{lemma}[theorem]{Lemma}

\newtheorem{definition}[theorem]{Definition}

\theoremstyle{remark}

\newcommand{\PrivateGGM}{\textsc{PACE-GGM}}

\newcommand{\SSP}{\textsc{SSP}}

\DeclareMathOperator*{\argmin}{argmin}
\DeclareMathOperator*{\argmax}{argmax}

\newcommand{\R}{\mathbb R}

\newcommand{\ds}[1]{}
\newcommand{\cf}[1]{}
\newcommand{\brett}[1]{}

\title{Private Adaptive Covariance Estimation via Gaussian Graphical Models}

\author{%
  Cecilia Ferrando\thanks{Contact: cferrando@cs.umass.edu}
  \quad Miguel Fuentes
  \quad Brett Mullins
  \quad Cameron Musco
  \quad Daniel Sheldon \\
  Manning College of Information and Computer Sciences \\
  University of Massachusetts Amherst
}

\begin{document}

\maketitle

\begin{abstract}
We propose \textsc{PACE-GGM}, a data-adaptive differentially private method for covariance estimation that concentrates its privacy budget on the most informative entries of the empirical covariance matrix, rather than perturbing all entries. 
This applies in the natural setting where the modeler supplies separate bounds for each variable, so that individual entries can be measured with less noise than the full matrix.
In each round, our method selects a poorly approximated entry, measures it using the Gaussian mechanism, and then reconstructs a full covariance matrix using a maximum-entropy reconstruction objective, leading to a Gaussian graphical model structure.
Experiments on diverse real-world datasets demonstrate consistent improvements in estimation error with respect to the Gaussian mechanism and other baselines, particularly in high-dimensional and low-to-moderate privacy regimes.
\end{abstract}

\section{Introduction}\label{sec:intro}

Covariance estimation is a fundamental primitive in statistics and machine learning, underpinning tasks such as regression, principal component analysis, and multivariate modeling. In settings involving sensitive data, differential privacy (DP) provides a rigorous framework for releasing covariance estimates while protecting individual-level information.
A canonical approach to private covariance estimation is to use the Gaussian mechanism to add noise to each of the $d(d+1)/2$ distinct entries of the empirical covariance matrix of $d$-dimensional data records, with sensitivity computed from an assumed bound on the $\ell_2$ norm of a data record.
This approach was pioneered in~\citep{dwork2014analyze} and is used as a core routine in many other algorithms~\citep{sheffet2017differentially, wang2018revisiting, alabi2020simple, biswas2020coinpress,  lin2024differentially}.

However, in many applications it is more natural to assume \emph{coordinate-wise} bounds rather than an a priori bound on the $\ell_2$ norm of a full data record, and in this setting 
measuring individual entries can be substantially cheaper, from a privacy perspective, than perturbing the entire matrix.
For example, assuming the $i$th coordinate of each data record has magnitude at most $1$, the sensitivity of an invidual entry of the covariance matrix is a factor of $d/\sqrt{2}$ smaller than that of the full matrix, so it can be measured more precisely than the full matrix for the same privacy budget. 
This observation raises a natural question: can we design an algorithm that selectively measures a subset of entries of the empirical covariance matrix and uses them to reconstruct a full matrix?
Selective measurement methods are widely used in differential privacy for query answering and synthetic data generation and often give state-of-the-art performance at realistic privacy budgets~\cite{mckenna2021hdmm, liu2021iterative, aydore2021differentially, zhang2021privsyn, cai2021data, mckenna2022aim,fuentes2026fast}.
Applying this principle to covariance estimation requires both an adaptive selection procedure and a principled reconstruction method.

We propose \emph{private adaptive covariance estimation via Gaussian graphical models} (\PrivateGGM{}) which iteratively selects the worst approximated entries, privately measures them via the Gaussian mechanism, and reconstructs the full covariance matrix using a new maximum-entropy reconstruction method. This connects to the classical covariance selection problem of \citet{dempster1972covariance}, yielding estimates with sparse precision matrices.
Empirically, \PrivateGGM{} outperforms the Gaussian mechanism and other existing methods, with gains most pronounced in the high-dimensional or low-to-moderate privacy budget regime.

\section{Related Work}
\label{sec:related}

\paragraph{Differentially private covariance estimation}
The canonical approach to private covariance estimation is the Gaussian mechanism, introduced by~\citet{dwork2014analyze}: a symmetric Gaussian noise matrix is added to the empirical covariance, with noise applied to all entries and noise scale calibrated to the data only through global sensitivity. 
The same method is often called ``sufficient statistic perturbation'' (\SSP{}) in private parametric statistical estimation, especially for linear regression~\citep{wang2018revisiting}, and we often use this name for brevity.
While simple and computationally efficient, \SSP{} is fundamentally data-independent: the same noise scale is used to perturb every entry regardless of its magnitude or importance, and the output is not guaranteed to be positive semidefinite (PSD).

\citet{biswas2020coinpress} propose CoinPress, which builds on~\citep{kamath2019privately} and assumes a conservative prior spectral bound on the population covariance
and iteratively clips data and shrinks this bound to reduce sensitivity, but still privatizes the full $d \times d$ matrix each iteration via the Gaussian mechanism.
\citet{wang2021differentially} study the high-dimensional sparse setting and propose DP-Thresholding: Gaussian noise is added to all entries of the empirical covariance, and a deterministic hard-thresholding step is applied in post-processing to exploit the assumed sparsity of the true covariance. The privacy budget is spent on the full matrix, with sparsity used only to reduce estimation error. 
\citet{dong2022differentially} propose SeparateCov and AdaptiveCov, which improve over \SSP{} in the high-dimensional regime by decoupling eigenvalue and eigenvector estimation.
AdaptiveCov further selects the best between two different estimators. Despite the structural decomposition, both algorithms still involve a step that privatizes the full covariance matrix, and the adaptivity operates at the level of estimator selection, not entry measurement.

In contrast, \PrivateGGM{} is the first method to privately and adaptively select which entries to measure based on the data. By concentrating the privacy budget on the entries that are most poorly approximated by the current estimate, and recovering the remaining entries through maximum-entropy reconstruction, \PrivateGGM{} avoids spending budget on entries that cannot be measured informatively.

\paragraph{Gaussian graphical models and covariance selection}
The maximum-entropy reconstruction at the heart of \PrivateGGM{} is directly connected to the classical theory of Gaussian graphical models~\citep{lauritzen1996graphical}. \citet{dempster1972covariance} showed that the maximum likelihood estimator of a Gaussian distribution subject to a fixed sparsity pattern $S$ on the precision matrix $K = \Sigma^{-1}$ is equivalent to the maximum-entropy PSD completion of the observed entries in $S$, with $K_{jk} = 0$ automatically for all $(j,k) \notin S$. 
This is the \emph{covariance selection} problem, and the solution is known to be a Gaussian graphical model with independence graph determined by $S$~\cite{lauritzen1996graphical}.
\PrivateGGM{} extends this classical framework to the private adaptive measurements setting where the algorithm selects $S$ privately and iteratively using the exponential mechanism, and observes entries with Gaussian noise rather than exactly. 

\section{Background}
\label{sec:dp-background}

 \ds{I changed to $X, X'$, for datasets since we use $X$ in the main paper. Also changed to use $\mathcal X$ for the data universe of a \emph{single} record and $\mathcal X^n$ for the space of datasets (previously $\mathcal D$ but not clearly defined). I changed the definition of neighboring datasets to the bounded DP model, which is what we use.}

We work under zero-concentrated differential privacy
(zCDP)~\citep{bun2016concentrated} in the bounded DP model, where $X \in \mathcal X^n$ denotes a dataset of $n$ records from the data universe $\mathcal X$.

\begin{definition}[Neighboring datasets]
Two datasets $X, X' \in \mathcal X^n$ are \emph{neighboring} ($X \sim X'$) if they differ in a single record.
\end{definition}

\begin{definition}[$\rho$-zCDP~\citep{bun2016concentrated}]
A randomized mechanism $\mathcal{M}$ satisfies $\rho$-zCDP if for all
$X \sim X'$ and all $\gamma > 1$,
\[
D_\gamma(\mathcal{M}(X)\|\mathcal{M}(X')) \leq \rho \gamma,
\]
where $D_\gamma(\mathcal{M}(X)\|\mathcal{M}(X'))$ is the $\gamma$-Rényi divergence between the distributions of $\mathcal{M}(X)$ and $\mathcal{M}(X')$.
\end{definition}

\begin{definition}[$\ell_2$ sensitivity]
The $\ell_2$ sensitivity of $f: \mathcal{X}^n \to \mathbb{R}^d$ is
$\Delta(f) = \max_{X \sim X'} \|f(X) - f(X')\|_2$.
\end{definition}

\begin{proposition}[Gaussian mechanism~\citep{bun2016concentrated}]
\label{prop:gaussian}
Let $f: \mathcal{X}^n \to \mathbb{R}^d$ be a function with $\ell_2$ sensitivity
$\Delta(f)$. The mechanism $\mathcal{M}(X) = f(X) + \mathcal{N}(0, \sigma^2 I_d)$
satisfies $\frac{\Delta(f)^2}{2\sigma^2}$-zCDP.
\end{proposition}

\begin{proposition}[Exponential mechanism~\citep{mcsherry2007mechanism,bun2016concentrated}]
\label{prop:exponential}
Let $\text{score}_r(X)$ be a quality score with sensitivity
$\Delta$. The mechanism that selects $r$ with probability proportional to
$\exp\!\bigl(\frac{\epsilon}{2\Delta}\text{score}_r(X)\bigr)$
satisfies $\frac{\epsilon^2}{8}$-zCDP.
\end{proposition}

\begin{proposition}[Composition and post-processing~\citep{bun2016concentrated,whitehouse2023fully}]
\label{prop:composition}
If $\mathcal{M}_1$ satisfies $\rho_1$-zCDP and $\mathcal{M}_2$ satisfies
$\rho_2$-zCDP, their (possibly adaptive) composition satisfies
$(\rho_1 + \rho_2)$-zCDP. Any post-processing of a $\rho$-zCDP mechanism
satisfies $\rho$-zCDP.
\end{proposition}

\begin{lemma}[zCDP to approximate DP~\citep{bun2016concentrated}]
\label{lem:conversion}
If $\mathcal{M}$ satisfies $\rho$-zCDP, then for any $\delta > 0$ it
satisfies $\bigl(\rho + 2\sqrt{\rho\ln(1/\delta)},\,\delta\bigr)$-DP.
\end{lemma}

\section{Method}

\subsection{Problem setup and motivation}


\paragraph*{Setup}
Assume we have a dataset $X \in \R^{n \times d}$ with rows 
$x^{(1)}, \ldots, x^{(n)}$ corresponding to individual records and a user-supplied coordinate-wise bound $|x_j| \leq B$ for each possible data record $x$ and coordinate $j$. 
The user must guarantee these bounds either using domain knowledge or by clipping the data or discarding records that violate them. 
If different bounds are supplied for each coordinate, the columns of $X$ can be rescaled to use a common bound $B$.
We focus on estimating the empirical second-moment matrix $\Sigma(X) = \frac{1}{n} X^\top X$ under $\rho$-zCDP. 
Throughout this work we assume the data is centered (or equivalently that the population mean is zero), so that $\Sigma(X)$ coincides with the empirical covariance matrix. 
More generally, one could privately estimate the mean and use the plug-in covariance estimator
$
\frac{1}{n} X^\top X - \mu\mu^\top,
$
but we omit this step for simplicity. 
While we focus on estimating the empirical covariance, such a routine is also useful for estimating the population covariance under a statistical model in the many cases where $\Sigma(X)$ is a natural estimator of the population covariance.

\paragraph*{Motivation for selective measurement}
\label{sec:sensitivity}
Our coordinate-wise bounding assumption is equivalent to the $\ell_\infty$ bound $\|x\|_\infty \leq B$.
This is natural whenever features are independently bounded, e.g., tabular datasets with clipped or normalized attributes, measurements with known reference ranges, or any setting where each feature is pre-processed to lie within a fixed interval (which after rescaling maps to $[-B, B]$).
In this setting, the sensitivity of a single entry of $\Sigma(X)$ is $2B^2/n$, while the sensitivity of the full matrix is $\sqrt{2}dB^2/n$ (see Appendix~\ref{app:sensitivity}), so a single entry can be measured with noise standard deviation a factor of $d/\sqrt{2}$ smaller than that required to measure the full matrix for the same privacy budget.
This motivates the idea of selectively measuring a subset of entries that are most important. 

Most prior covariance estimation methods assume an $\ell_2$ bound $\|x\|_2 \leq C$ on the data. 
In this case, the sensitivity of a single entry is about the same as the full matrix $(C^2/n$ vs. $\sqrt{2}C^2/n$), so it does not appear helpful to measure a subset of entries. 
To understand the apparent discrepancy, suppose the modeler knows only $\|x\|_\infty \leq B$ and needs to derive an $\ell_2$ bound. 
Since each entry of $x$ could have magnitude $B$, the tightest bound is $\|x\|_2 \leq C = \sqrt{d}B$: this yields the same sensitivity of $\sqrt{2}dB^2/n$ for the full matrix but a now loose bound of $dB^2/n$ for a single entry.

To summarize, if the user starts by knowing only coordinate-wise bounds, which is a very natural setting, then measuring individual entries can save privacy budget relative to measuring the full matrix. 
However, this possibility may be obscured if the user uses the $\ell_\infty$ bound to convert to an $\ell_2$ bound. 
On the other hand, if the user truly knows a tight $\ell_2$ bound, e.g., $\|x\|_2 \leq C$ where $C$ is not much bigger than the $\ell_\infty$ bound $\|x\|_\infty \leq B$---informally, this requires prior knowledge that not too many entries of $x$ can be large simultaneously---then measuring individual entries is not likely to be helpful compared to measuring the full matrix. 



\subsection{The \PrivateGGM{} algorithm}
Our proposed method for \emph{private adaptive covariance estimation via Gaussian graphical models} (\PrivateGGM{}) is shown in Algorithm~\ref{alg:main-algorithm}. 
\PrivateGGM{} is based on the \emph{select-measure-reconstruct} paradigm used throughout the adaptive query release and synthetic data generation literature~\citep{mckenna2021hdmm, liu2021iterative, aydore2021differentially, zhang2021privsyn, cai2021data, mckenna2022aim,fuentes2026fast}, and in particular on AIM (\emph{adaptive iterative mechanism}), a state-of-the-art method for synthetic data and marginal query answering~\cite{mckenna2022aim}.
\PrivateGGM{} initially uses a portion of the privacy budget to measure all diagonal elements of $\Sigma$ and form an initial estimate $\hat \Sigma^{(0)}$ (Algorithm~\ref{alg:initialization}).
Throughout the algorithm, it maintains a set of noisy measurements $\{y_{jk}\}$ of matrix entries together with the values $\{\lambda_{jk}\}$ representing the precision (inverse variance) of each $y_{jk}$.
It then iteratively follows the select-measure-reconstruct framework. 
In each iteration, it first \textbf{selects} an entry $\Sigma_{jk}$ that is poorly approximated by the current estimate using the exponential mechanism.
It then \textbf{measures} the selected entry with the Gaussian mechanism and updates the precision-weighted average measurement $y_{jk}$ and precision $\lambda_{jk}$ for that entry. 
Finally, it \textbf{reconstructs} a new covariance estimate $\hat\Sigma^{(t)}$ by solving an optimization problem to find the maximum entropy PSD matrix among all minimizers of a reconstruction error that is equivalent to the negative log-likelihood of the observed measurements. 

The sensitive data $X$ is only accessed in the select and measure steps.
These steps use the exponential and Gaussian mechanisms together with a zCDP composition analysis similar to the one in AIM~\citep{mckenna2022aim} to prove that the whole procedure satisfies $\rho$-zCDP. 
We return to the details in Section~\ref{sec:privacy-analysis} after describing our maximum entropy reconstruction method, which is a key innovation to enable adaptive covariance estimation. 

\setlength{\textfloatsep}{6pt}
\setlength{\floatsep}{6pt}
\setlength{\intextsep}{6pt}
\vspace{2mm}

\begin{algorithm}[h]
  \caption{
    \PrivateGGM{}
  }
  \label{alg:main-algorithm}
  \begin{algorithmic}[1]
  \Require Data $X \in \R^{n \times d}$, max rounds $T$, budget $\rho$, coordinate-wise bound $B$, parameters $\alpha, \beta$ 
  \Ensure Private covariance estimate $\hat{\Sigma}$

  \State Per-entry sensitivity $\Delta \gets 2B^2/n$
  \State Initialize measurements $\{y_{jk}\}$, precisions $\{\lambda_{jk}\}$, estimate $\hat \Sigma^{(0)}$ using Algorithm~\ref{alg:initialization} with budget $\rho_{\text{diag}} = \alpha \rho$.
  \State $\rho_{\text{used}} \gets \rho_{\text{diag}}$ \quad
  $\rho_{\text{per-round}} \gets (\rho - \rho_{\text{used}})/T$ \quad 
  $\rho_{\text{sel}} \gets \beta \rho_{\text{per-round}}$ \quad 
  $\rho_{\text{meas}} \gets (1-\beta)\rho_{\text{per-round}}$
  \State $t \gets 0$
  \While{$\rho_{\text{used}} < \rho$}
     \State $t \gets t + 1$
     \State \textbf{Select} entry $(j, k)$ via exponential mechanism with $\rho_{\text{sel}}$ and worst-approx. score function:
     \[
      \Pr(\text{select\ } (j, k)) \propto \exp\left(\frac{\epsilon}{2\Delta} \cdot \text{err}(j, k)\right), \quad
     \epsilon = \sqrt{8\rho_{\text{sel}}}, \quad
     \text{err}(j, k) = |\Sigma_{jk}(X) - \hat{\Sigma}^{(t-1)}_{jk}|
     \]
    \State \textbf{Measure} $\Sigma_{jk}$ with budget $\rho_{\text{meas}}$ and update precision-weighted average measurement:
    \[
      z_{jk} \gets \Sigma_{jk}(X) + \mathcal{N}(0, \sigma^2_t), \quad \sigma^2_t = \Delta^2 / (2\rho_{\text{meas}})
    \]
    \[
      y_{jk} \gets (\lambda_{jk}\, y_{jk} + \sigma^{-2}_t\, z_{jk}) / (\lambda_{jk} + \sigma^{-2}_t),
      \qquad
      \lambda_{jk} \gets \lambda_{jk} + \sigma^{-2}_t
    \]
    \State \textbf{Reconstruct} covariance matrix via maximum-entropy reconstruction:
    \[
      \hat{\Sigma}^{(t)} \gets \argmax\{\log|\Sigma| : \Sigma \in \arg\min_{W \succeq 0} \sum_{j \ge k} \lambda_{jk}\,(W_{jk} - y_{jk})^2\}
    \]
    \State $\rho_{\text{used}} \gets \rho_{\text{used}} + \rho_{\text{sel}} + \rho_{\text{meas}}$
    \State Anneal budgets using Algorithm~\ref{alg:budget-annealing}.
  \EndWhile
  \State \Return $\hat{\Sigma}^{(t)}$
  \end{algorithmic}
\end{algorithm}

\begin{figure}[h]
  \centering
  \noindent
  \begin{minipage}{0.40\textwidth}
      \begin{algorithm}[H]
        \caption{Initialization}\label{alg:initialization}
        \begin{algorithmic}[1]
          \State $\lambda_{jk} \gets 0$, $y_{jk} \gets 0$ for all $j \geq k$
          \State $\sigma^2 = \Delta^2 / (2\rho_{\text{diag}}/d)$
          \For{$j = 1$ to $d$}
            \State $y_{jj} \gets \Sigma_{jj}(X) + \mathcal{N}(0, \sigma^2)$
            \State $\lambda_{jj} \gets \sigma^{-2}$
          \EndFor
          \State $\hat{\Sigma}^{(0)} \gets \operatorname{diag}(y_{11} \vee 0, \ldots, y_{dd} \vee 0)$
          \State \Return $\{y_{jk}\}$, $\{\lambda_{jk}\}$, $\hat{\Sigma}^{(0)}$
        \end{algorithmic} 
      \end{algorithm}
  \end{minipage}
  \hfill
  \begin{minipage}{0.56\textwidth}
      \begin{algorithm}[H]
        \caption{Budget annealing}\label{alg:budget-annealing}
        \begin{algorithmic}[1]
          \If{$|\hat{\Sigma}_{jk}^{(t)} - \hat{\Sigma}^{(t-1)}_{jk}| \le \sqrt{2/\pi}\,\sigma_t$} 
          \State $\rho_{\text{sel}} \gets 2\,\rho_{\text{sel}}$
          \quad $\rho_{\text{meas}} \gets 4\,\rho_{\text{meas}}$
          \EndIf
          \If{$(\rho - \rho_{\text{used}}) < 2(\rho_{\text{sel}} + \rho_{\text{meas}})$}
          \State $\rho_{\text{sel}} \gets \beta(\rho - \rho_{\text{used}})$
          \quad $\rho_{\text{meas}} \gets (1-\beta)(\rho - \rho_{\text{used}})$
          \EndIf
        \end{algorithmic}
      \end{algorithm}
  \end{minipage}
\end{figure}

\subsection{Maximum-entropy covariance reconstruction}

\textbf{General problem}
Suppose we are given noisy observations of some entries of a PSD matrix $\Sigma$, including all of the diagonal entries:
\begin{equation}
  \label{eq:measurement-model}
  y_{jk} = \Sigma_{jk} + \mathcal{N}(0, \tau_{jk}^2), \quad (j,k) \in S
\end{equation}
where $S$ is a set of ordered pairs $(j,k)$ with $j \ge k$ and $(j,j) \in S$ for all $j$. 
This setup assumes that the diagonal entries are always observed, and additionally some entries in the lower triangle are observed; the upper triangle is determined by symmetry.
We further assume that each entry is measured at most once; we show below that this is without loss of generality. 

Our reconstruction objective is to solve the following bilevel optimization problem:
\begin{align}\label{eq:objective}
\max\log|\Sigma| \quad \text{s.t.} \quad
\Sigma \in \arg\min_{W \succeq 0} \mathcal L(W), \quad
\mathcal L(W) = \sum_{(j,k)\in S} 
\frac{1}{2 \tau_{jk}^2} (W_{jk} - y_{jk})^2.
\end{align}
The inner objective $\mathcal L(W)$ is the negative log-likelihood of the observations $\{y_{jk}\}_{(j,k) \in S}$ under covariance matrix $W$ plus a constant.
Thus, the inner optimization seeks the PSD matrix that maximizes the log-likelihood of the measurements. 
The outer objective selects the minimizer of $\mathcal L(W)$ with the largest log-determinant.
This is equivalent to finding the maximum-entropy distribution whose covariance matrix minimizes $\mathcal L(W)$: for any $W \succeq 0$, the distribution $\mathcal{N}(0, W)$ has maximum differential entropy over all distributions on $\mathbb R^d$ with the same covariance, and its entropy is $\frac{1}{2}\log|W|$ plus a constant~\citep{cover1991elements}.
Entropy is a measure of smoothness of a distribution and is often used to select the simplest explanation matching a set of observations~\cite{dempster1972covariance}.

Our problem is closely connected to Dempster's \emph{covariance selection} framework~\citep{dempster1972covariance} and Gaussian graphical models (GGMs)~\citep{lauritzen1996graphical}.
If $\tau_{jk} \to 0$ for all $(j, k) \in S$ the inner objective requires $W_{jk} \to y_{jk}$ and the problem becomes the covariance selection problem, which is to find the maximum entropy completion of a partially specified covariance matrix.
Our problem is a generalization where observations are noisy, so the inner optimization seeks to maximize likelihood rather than constrain $W$ to exactly match observations.  
Dempster showed that the covariance selection problem has a unique solution with a sparse inverse. 
Our problem shares this structure:
\begin{proposition}
  \label{prop:sparse-inverse}
  Suppose $\arg\min_{W \succeq 0} \mathcal L(W)$ contains a positive definite matrix.
  Then Problem~\ref{eq:objective} has a unique solution $\Sigma^*$ and $K^* = (\Sigma^*)^{-1}$ satisfies $K^*_{jk} = 0$ for $(j, k), (k, j) \notin S$.
\end{proposition}
\begin{proof}[Proof sketch]
It can be shown that there are unique values $(a_{jk})_{(j,k) \in S}$ such that $W$ minimizes $\mathcal L(W)$ if and only if $W_{jk} = a_{jk}$ for all $(j, k) \in S$. The result then follows from \citep{dempster1972covariance}.
\end{proof}
The full proof is given in Appendix~\ref{app:properties}.
If $\arg\min_{W \succeq 0} \mathcal L(W)$ does not contain a positive definite matrix, the problem is ill posed and the optimal value is $-\infty$; we discuss further below and in \ref{app:ill-posedness}.

Proposition~\ref{prop:sparse-inverse} is a useful structural property and says that the solution is the covariance matrix of a Gaussian graphical model, which satisfies certain conditional independence properties~\citep{lauritzen1996graphical}. 
For example, if $X \sim \mathcal{N}(0, \Sigma^*)$, then $X_j$ and $X_k$ are conditionally indpendent given the other variables whenever $K^*_{jk} = 0$, i.e., whenever the entries $(j, k)$ and $(k, j)$ are not observed. 
In practice, our optimizer will solve Problem~\ref{eq:objective} to within a numerical tolerance such that the inverse is not exactly sparse, though it is likely possible to design optimizers that enforce sparsity exactly.

\paragraph{Application to \PrivateGGM{}}

It is easy to see that whenever \PrivateGGM{} measures each entry at most once, the measurements $\{y_{jk}\}$ fit the model in Equation~\eqref{eq:measurement-model} with $\tau_{jk}^2 = \lambda_{jk}^{-1}$ and $S = \{(j, k): \lambda_{jk} > 0\}$, giving the objective in the algorithm:
\[
\mathcal L(W) = \sum_{(j,k)\in S} \lambda_{jk}\,(W_{jk} - y_{jk})^2, \quad S = \{(j, k) : \lambda_{jk} > 0\}.
\]
Now we will show that this objective is also correct when $y_{jk}$ is the inverse-variance weighted average of multiple measurements and $\lambda_{jk}$ is the combined precision.

\begin{proposition}
  \label{prop:ivw}
Let $\mathcal L(W) = \sum_{t=1}^T \frac{1}{2\sigma_t^2} (W_{j_t k_t} - z_t)^2$ be the log-likelihood of a sequence of noisy measurements $z_t = \Sigma_{j_t k_t} + \mathcal{N}(0, \sigma^2_t)$ for $t = 1, \ldots, T$. This can be rewritten as
\[
\mathcal L(W) = \sum_{(j,k) \in S} \lambda_{jk}\,(W_{jk} - y_{jk})^2 + C,
\]
where 
$\mathbb T_{jk} = \{t : (j_t, k_t) = (j, k)\}$,\,
$\lambda_{jk} = \sum_{t \in \mathbb{T}_{jk}} \sigma_t^{-2}$,\,
$y_{jk} = \lambda_{jk}^{-1}\sum_{t \in \mathbb{T}_{jk}} \sigma_t^{-2} z_t$,\,
$S = \{(j, k) : \mathbb{T}_{jk} \neq \emptyset\}$,
and $C$ is a constant independent of $W$.
\end{proposition}
The proof is in Appendix~\ref{app:ivw}.
\PrivateGGM{} maintains the precision-weighted average $y_{jk}$ and combined precision $\lambda_{jk}$ for each entry and updates them with each new measurement.


\subsection{Solving the reconstruction problem}
\label{sec:optimization}

\paragraph*{Optimization approach} 
There are multiple possible approaches to solve Problem~\ref{eq:objective}. 
One option is a two-stage approach suggested by the proof of Proposition~\ref{prop:sparse-inverse}: first solve the inner problem to find any minimizer $W^*$, extract the unique values $a_{jk} = W^*_{jk}$ for $(j, k) \in S$, and then solve the covariance selection problem with $(a_{jk})_{(j,k) \in S}$ as the observed entries.
However, we found this to be unstable in practice: with noisy enough observations, it is often the case that \emph{every} matrix in $\argmin_{W \succeq 0} \mathcal L(W)$ is singular and therefore the covariance selection problem is ill posed (for details, see the discussion in Appendix~\ref{app:ill-posedness}).

Instead, we propose an interior point method where $\Sigma$ is parameterized by a Cholesky factor (\textsc{IPM-Cholesky}). 
The interior point method solves a sequence of strictly convex barrier-penalized problems and provides a clean way to handle ill posed problems.
For $\mu > 0$, define
\begin{equation}
  \label{eq:barrier-problem}
\Sigma^*_\mu = \argmin_{W \succeq 0} (\mathcal L(W) - \mu \log|W|).
\end{equation}
It can be shown that $\Sigma^*_\mu$ is uniquely determined for each $\mu > 0$ and converges to a unique $\Sigma^* \in \argmin_{W \succeq 0} \mathcal L(W)$ as $\mu \to 0$~\citep{boyd2004convex}.
When the condition of Proposition~\ref{prop:sparse-inverse} holds, $\Sigma^*$ is the unique positive definite solution to Problem~\ref{eq:objective}. 
When the condition does not hold, all minimizers have entropy of $-\infty$; selecting $\Sigma^*$ is well motivated as it is the limit of a sequence of entropy regularized solutions. 

It remains to solve the barrier-penalized problem~\eqref{eq:barrier-problem}.
To handle the PSD constraint, we parameterize $\Sigma$ via its Cholesky decomposition $\Sigma = LL^\top$, where $L$ is lower triangular with nonnegative diagonal entries.
The log-determinant becomes $\log|\Sigma| = 2\sum_{j=1}^d \log L_{jj}$.
Then, for a sequence of geometrically decreasing values of $\mu$, we minimize the barrier-penalized objective
\begin{align}
\varphi_\mu(L) = \sum_{(j,k)\in S}\frac{((LL^\top)_{jk} - y_{jk})^2}{2\tau_{jk}^2}
- 2\mu\sum_{j=1}^d \log L_{jj}
\end{align}
via L-BFGS-B on $\varphi_\mu$ with bound constraints $L_{jj} \geq 0$. 
A downside is that $\varphi_\mu(L)$ is not convex in $L$, so global convergence is not guaranteed. 
However, we found performance to be excellent in practice, and this method is closely related to Burer-Monteiro methods for semidefinite programming~\citep{burer2003nonlinear}, which are widely reported to perform well in practice and provably converge under certain conditions~\citep{journee2010low}.
Notably, the per-iteration cost to compute $\varphi_\mu(L)$ and its gradient is only $O(|S|\, d)$---dominated by the time to compute $(LL^\top)_{jk}$ for each $(j,k) \in S$---and, if $S$ is sparse, avoids the typical $\Theta(d^3)$ per-iteration cost for dealing with PSD matrices.

\paragraph{Sparsity-aware optimization}
The set $S$ determines a graph $G$ on the vertex set $\{1, \ldots, d\}$ with an edge $(j,k)$ whenever $(j,k) \in S$. 
Problem~\ref{eq:objective} is closely related to sparse matrix nearness~\citep{sun2015decomposition} and covariance selection~\citep{dahl2008covariance} problems.
By embedding $G$ in a chordal graph with maximal cliques of size $d_1, \ldots, d_m$, relevant linear algebraic operations can be performed in time $O(\sum_{i=1}^m d_i^3)$ via clique trees rather than $O(d^3)$ time needed for dense matrices, which is beneficial when the cliques are small.
We explored such methods and found them promising for very large problems, but they required non-trivial adaptations and yielded at most modest gains at the scales considered here ($d \leq 260$), so we leave further exploration to future work.

On the other hand, since \PrivateGGM{} starts with only diagonal measurements and then adds a single entry per round, the measurement sets $S$ are often extremely sparse, and we found that a simple decomposition based on connected components yielded significant gains. 
Suppose $G$ has connected components $V_1, \ldots, V_K$. 
Since $S$ contains no edges between components, $\mathcal L(W)$ depends only on the diagonal blocks $W_{V_i, V_i}$. 
Among feasible $W$, Fischer's inequality $|W| \leq \prod_i |W_{V_i, V_i}|$ (with equality iff $W$ is block diagonal) implies the maximum of $\log|W|$ is attained at a block-diagonal matrix. 
Hence $\Sigma^*$ is block diagonal and each block solves an independent instance of Problem~\ref{eq:objective} on $V_i$. 
In the extreme case where $S$ contains only the diagonal entries, each block is a single vertex, so the solution is simply the diagonal matrix with entries $y_{jj} \vee 0$ (cf. $\hat \Sigma^{(0)}$ in Algorithm~\ref{alg:initialization}).

\subsection{Budget annealing}
A key challenge in adaptive measurement algorithms is deciding how many measurements to make vs. relying on reconstruction to fill in unmeasured quantities.
A good choice typically depends on the privacy budget, the problem size, and data properties.
A key contribution of AIM was to introduce a budget annealing strategy to address this challenge~\citep{mckenna2022aim}, which we adapt here (Algorithm~\ref{alg:budget-annealing}). 
We start with a per-round budget that would allow running for up to $T$ rounds, where $T$ is a generous bound set by the user.
Then, in Line 1 of Algorithm~\ref{alg:budget-annealing}, the procedure checks how much the estimate $\hat \Sigma_{jk}$ changes compared to the previous round after $\Sigma_{jk}$ is measured.
If the change is small relative to the noise standard deviation $\sigma_t$ (the quantity $\sqrt{2/\pi}\,\sigma_t$ is $\mathbb E[|\xi|]$ for $\xi \sim \mathcal N(0, \sigma_t^2$)), 
it indicates the current noise is too large to provide informative measurements, so the per-round privacy budgets are increased (Line 2). 
This will lead to fewer but more informative future measurements. 
Lines 4--6 of Algorithm~\ref{alg:budget-annealing} are a simple measure to ensure the privacy budget is fully exhausted in the final round.


\subsection{Privacy guarantee}
\label{sec:privacy-analysis}

\begin{theorem}[Privacy of Algorithm~\ref{alg:main-algorithm}]
\label{thm:privacyproof}
Algorithm~\ref{alg:main-algorithm} satisfies $\rho$-zCDP.
\end{theorem}

\begin{proof}
We analyze each phase separately and apply zCDP composition
(Proposition~\ref{prop:composition}).
Recall that the sensitivity of each entry of $\Sigma$ is $\Delta = 2B^2/n$ (Appendix~\ref{app:sensitivity}).

\textit{Phase 1.} Each diagonal entry is released via the Gaussian mechanism
with sensitivity $\Delta$ and per-entry budget $\rho_{\mathrm{diag}}/d$, satisfying $(\rho_{\mathrm{diag}}/d)$-zCDP per entry. 
By composition over $d$ independent releases, Phase~1 satisfies
$\rho_{\mathrm{diag}}$-zCDP where $\rho_{\mathrm{diag}} = \alpha\rho$.

\textit{Phase 2.} Each round $t$ consists of (i) an exponential mechanism
selection step satisfying $\rho_{\mathrm{sel}}^{(t)}$-zCDP and (ii) a
Gaussian mechanism measurement step satisfying
$\rho_{\mathrm{meas}}^{(t)}$-zCDP. Since the selection at round $t{+}1$
depends on the output of round $t$, this is an adaptive composition;
by~\citet{whitehouse2023fully}, each round satisfies
$(\rho_{\mathrm{sel}}^{(t)} + \rho_{\mathrm{meas}}^{(t)})$-zCDP.
The algorithm tracks cumulative expenditure $\rho_{\mathrm{used}}$ and
enforces
$\sum_t(\rho_{\mathrm{sel}}^{(t)} + \rho_{\mathrm{meas}}^{(t)})
\leq (1-\alpha)\rho$.

Therefore the total privacy cost across both phases is at most $\rho$.
The maximum entropy reconstruction step is post-processing and incurs no additional privacy cost (Proposition~\ref{prop:composition}).
\end{proof}

\section{Experiments}\label{sec:experiments}

We evaluate \PrivateGGM{} on real-world datasets to assess its ability to accurately estimate covariance matrices under zCDP across a range of privacy budgets compared to non-adaptive baselines.
We work with diverse datasets with dimensionality ranging from $d=6$ to $d=260$. 
Appendix~\ref{app:datasets} provides detailed dataset descriptions, including key structural properties in Table~\ref{tab:dataset-structure}. 
For each dataset, we remove rows with missing values,
and rescale all features to $[-1,1]$ coordinate-wise, so that the $\ell_\infty$-ball
constraint $\|x\|_\infty \leq B = 1$ holds.
To isolate the covariance estimation component, we assume the mean is known and use it to center the data.
In practice, part of the privacy budget would be allocated to mean estimation.\cf{cut the private mean part here}
The datasets have a wide range of empirical covariance structures: the fraction of the squared Frobenius norm of $\Sigma$ that comes from off-diagonal entries ($e_{\mathrm{off}}$) varies from
near-zero for \texttt{adult} ($0.018$) and \texttt{thyroid\_ann} ($0.026$), which are approximately
diagonal, to $0.89$ for \texttt{communities\_crime}, where off-diagonal structure dominates.
This diversity helps investigate estimation under varying underlying structure. 



We evaluate \PrivateGGM{} over privacy budgets
$
\rho \in \{10^{-4}, 10^{-3}, 10^{-2}, 10^{-1}, 1, 2, 10\}
$.
We set $\alpha = 0.3$ (fraction of budget for diagonal initialization) and $\beta = 0.5$ (fraction of per-round budget for entry selection). An ablation study is included in Appendix~\ref{app:hyperparameters}.
The maximum number of rounds is set to $T = \lfloor 2 \cdot d(d-1)/2 \rfloor$, allowing up to twice as many measurements as there are
off-diagonal entries.
Remeasurement of previously measured entries, including the diagonal entries, is allowed.
All experiments use the \textsc{IPM-Cholesky} solver (Algorithm~\ref{alg:ipm}, Appendix~\ref{app:ipm}).
Each configuration is repeated for 10 independent trials using pre-specified random seeds.


\paragraph*{Error metrics}
We report two complementary metrics: (i) \textit{Mahalanobis error}
$\|\Sigma^{-1/2}\hat{\Sigma}\Sigma^{-1/2} - I\|_F$, which measures
shape accuracy by normalizing entry-wise error relative to the true
covariance structure and is invariant to affine transformations; and
(ii) \textit{Frobenius error} $\|\hat{\Sigma} - \Sigma\|_F$, which
measures entry-wise accuracy directly.




    

\paragraph*{Baselines}
We compare \PrivateGGM{} against several baselines. 
\textit{SSP} (Sufficient Statistic Perturbation) applies the Gaussian mechanism to all $d(d+1)/2$ entries of $\Sigma$ using the full budget $\rho$ and full-matrix sensitivity~\citep{dwork2014analyze}. 
\textit{AdaptiveCov}~\citep{dong2022differentially} selects between an \SSP{} estimator (GaussCov) and a spectral approach (SeparateCov) that privatizes eigenvalues and eigenvectors via Gaussian mechanisms and reconstructs the covariance; here, adaptivity refers to mechanism selection rather than entry-wise selection.
Finally, a \textit{diagonal-only baseline} allocates the entire budget to the $d$ diagonal entries (with $\rho/d$ per entry), sets off-diagonal entries to zero, and projects the result to the PSD cone, illustrating the cost of ignoring correlations (Appendix~\ref{app:plot-all})\footnote{
  CoinPress~\citep{biswas2020coinpress} is another well known private covariance estimator, where the main idea is to iteratively refine bounds on the scale of covariance matrix when only a coarse upper bound is available, so the empirical covariance can be measured (using \SSP{}) with less noise. 
  In preliminary experiments, CoinPress performed poorly in our setup.
  However, it is essentially solving an orthogonal problem to ours (finding the right scale for clipping data), which may be less relevant in our setting where we already provide tight $\ell_\infty$ bounds on the data, so we do not include it in our empirical comparisons.
}. 

\subsection{Results}\label{sec:paceggm-mainresults}

Figure~\ref{fig:plot-main} shows results for 5 of the 14 datasets across $d$. Appendix~\ref{app:plot-all} includes results for all datasets. \PrivateGGM{} overall outperforms \SSP{} and often achieves lower error than AdaptiveCov on both metrics. This illustrates that \PrivateGGM{} is competitive with state-of-the-art DP covariance estimators, often with the lowest error in the high-privacy (low $\rho$) regime, and as dimension increases. 

Two additional findings are reported in Appendix~\ref{app:solver}. 
First, we show that the maximum-entropy criterion consistently selects a high quality solution among all optimizers of the reconstruction objective $\mathcal L(W)$, while simply minimizing the objective via projected gradient descent (PGD) yields a solution that is sensitive to initialization and may be much worse than the maximum-entropy solution. Second, we show that exploiting sparsity in the measurement graph via connected components accelerates optimization without affecting estimation accuracy.

Table~\ref{tab:measurements_main} reports the average number of off-diagonal entries 
measured (lower triangle) across datasets and $\rho$ values. Three patterns emerge:

\textbf{Measurements grow with $\rho$.} As privacy budget increases, 
\PrivateGGM{} measures more off-diagonal entries across all datasets. At 
$\rho=10^{-4}$, essentially only the diagonal and a handful of off-diagonal 
entries are ever selected; as $\rho$ increases, the algorithm makes progressively more measurements.

\textbf{Sparsity grows with $d$.}
As $d$ increases, \PrivateGGM{} measures a very small fraction of 
available off-diagonal entries. At $\rho=10$, the fraction ranges from $87\%$ 
for \texttt{adult} ($d=6$, 12 out of 15 off-diagonal pairs) down to $0.1\%$ 
for \texttt{madeline} ($d=260$, 39 out of 33{,}930). This sparsity shows how \PrivateGGM{} concentrates the privacy budget on the entries that both carry significant information and can be measured informatively.
The larger $d$, the more this selective allocation matters, which explains why the advantage over \SSP{} grows with dimensionality.

\textbf{Measurements concentrate where covariance is denser.}
The number of selected entries is also governed by the density of the 
underlying covariance. We quantify density 
via $e_{\mathrm{off}} = \|\Sigma_{\mathrm{off}}\|_F^2 / \|\Sigma\|_F^2$, 
the fraction of squared Frobenius norm from off-diagonal entries.
At $\rho=10$, \texttt{communities\_crime} ($d=102$, 
$e_{\mathrm{off}}=89.5\%$) requires 255 off-diagonal measurements, more 
than six times as many as \texttt{madeline} ($d=260$, 
$e_{\mathrm{off}}=24.8\%$) despite having less than half the dimension. 
\PrivateGGM{} allocates more measurements where $e_{\mathrm{off}}$ is higher,
which demonstrates that it adapts to covariance structure and not just dimension.
Appendix~\ref{app:plot-all} compares \PrivateGGM{} to a diagonal-only (independent) baselines that only privatizes the diagonal, ignoring off-diagonal correlations: as expected, it is competitive with \PrivateGGM{} only on datasets with nearly diagonal covariance.

Together, these results show that \PrivateGGM{} often improves upon data-independent methods by measuring entries adaptively with respect to the dimension, budget, and underlying covariance structure.


\begin{figure}[ht]
\centering
\includegraphics[width=0.85\linewidth]{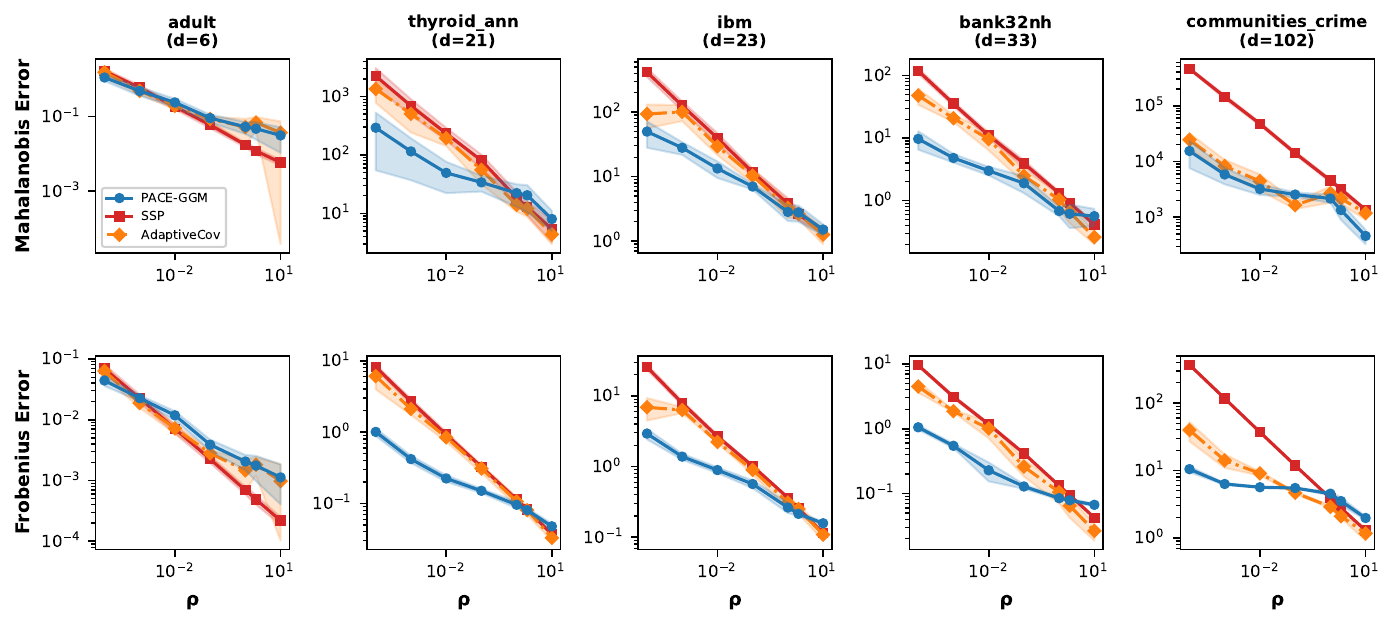}
\caption{Mahalanobis error (top row) and Frobenius error (bottom row) for a selection of datasets across $d$. \PrivateGGM{} and baselines across privacy budgets $\rho$, with $\alpha=0.3$, $\beta=0.5$. Datasets are ordered by dimensionality $d$ across three groups.}
\label{fig:plot-main}
\end{figure}

\begin{table}[t]
\centering
\resizebox{\linewidth}{!}{%
\begin{tabular}{llrrrrrrrr}
\toprule
Dataset & $d$ & $\frac{d(d+1)}{2}$ & $\rho=10^{-4}$ & $\rho=10^{-3}$ & $\rho=10^{-2}$ & $\rho=10^{-1}$ & $\rho=1$ & $\rho=2$ & $\rho=10$ \\
\midrule
adult                 & 6   & 21     & $6+2.6$   & $6+2.8$   & $6+3.4$   & $6+7.4$    & $6+11.6$   & $6+12.3$   & $6+12.2$   \\
thyroid\_ann          & 21  & 231    & $21+5.0$  & $21+5.1$  & $21+6.2$  & $21+8.6$   & $21+19.0$  & $21+24.6$  & $21+42.5$  \\
ibm                   & 23  & 276    & $23+5.4$  & $23+5.8$  & $23+7.1$  & $23+11.6$  & $23+22.2$  & $23+28.4$  & $23+38.1$  \\
bank32nh              & 32  & 528    & $32+6.1$  & $32+8.1$  & $32+10.9$ & $32+12.1$  & $32+18.4$  & $32+25.5$  & $32+49.7$  \\
communities\_crime    & 102 & 5253   & $102+7.7$ & $102+8.8$ & $102+11.0$& $102+23.3$ & $102+97.5$ & $102+130.2$& $102+255.3$\\
\bottomrule
\end{tabular}%
}
\caption{Mean number of measurements per trial (diagonal$+$off-diagonal). The column $d(d+1)/2$ gives the total number of unique entries in the covariance matrix.
}
\label{tab:measurements_main}
\end{table}

\section{Discussion}
\label{sec:discussion}

\paragraph{Limitations}
\PrivateGGM{} assumes continuous data with coordinate-wise bounds $\|x\|_\infty \leq B$.
As discussed earlier, if the modeler knows tight $\ell_2$ bounds that are not much bigger than $B$, there may be little advantage of measuring indvidual entries.
Performance on heavy-tailed or strongly non-Gaussian data has not been evaluated.
The per-round cost is $O(d^3)$ in the worst case, which may become limiting for very large $d$. 
The two hyperparameters $\alpha$ and $\beta$ have robust defaults
($\alpha{=}0.3$, $\beta{=}0.5$) and the method is empirically insensitive to their
exact values in $[0.1,0.5]\times[0.1,0.5]$ (see Appendix~\ref{app:hyperparameters}), but no fully private selection procedure
is provided.
Finally, our experimental evaluation focuses on \SSP{} and AdaptiveCov;
comparisons to additional baselines are left for future work.

\paragraph{Future directions}
Directions for future work include investigating theoretical bounds on the error of the adaptive selection procedure and its performance for estimating of population parameters of different statistical models. 
Computationally, a promising avenue is to develop a solver that uses chordal embeddings of the sparsity graph to reduce running times relative to the dense solver; this could allow scaling to very large $d$ with sparse measurements.

\section{Conclusion}
\label{sec:conclusion}

We introduced \PrivateGGM{}, a differentially private empirical covariance estimator
that adaptively selects which entries to measure based on the data.
By combining the exponential mechanism for entry selection,
the Gaussian mechanism for measurement, inverse-variance weighting for
repeated measurements, and maximum-entropy PSD reconstruction,
\PrivateGGM{} concentrates its privacy budget where most needed rather than
spreading noise uniformly across all $d(d+1)/2$ entries.
The $\ell_\infty$-ball assumption is natural for tabular data with clipped or
normalized features and enables our adaptive approach, since under this assumption
the single-entry sensitivity is a factor of $d/\sqrt{2}$ smaller than full-matrix methods, allowing a method to measure a subset of entries much more accurately than the full matrix.
We proved that the algorithm satisfies $\rho$-zCDP via adaptive composition and
demonstrated consistent improvements over non-adaptive baselines across fourteen real-world datasets spanning $d \in [6, 260]$, with gains most
pronounced in the high-dimensional, low-to-moderate budget regime that is most
demanding in practice.
\PrivateGGM{} is suitable as a building block for continuous synthetic data generation, filling a gap left by marginal-based methods that target discrete or discretized variables.

\bibliographystyle{plainnat}
\bibliography{pace-ggm}

\appendix

\section{Sensitivity analysis}\label{app:sensitivity}

In this Appendix, we derive the sensitivity bounds for the matrix $\Sigma(X) = \frac{1}{n} \sum_{i=1}^n x^{(i)}x^{(i)T}$ under $\ell_2$ and $\ell_\infty$ ball bounding assumptions.
We consider separartely the sensitivity of the full matrix, and a single entry.
Table~\ref{tab:cov_sensitivity} summarizes the sensitivity bounds.

\begin{table}[H]
\centering
\caption{Sensitivity bounds for the matrix $\Sigma(X) = \frac{1}{n} \sum_{i=1}^n x^{(i)}x^{(i)T}$ assuming either $\mathcal X = \{x \in \mathbb R^d: \|x\|_\infty \leq B\}$ (i.e., $\ell_\infty$ bound) or $\mathcal X = \{x \in \mathbb R^d: \|x\|_2 \leq C\}$ (i.e., $\ell_2$ bound).}
\label{tab:cov_sensitivity}
\renewcommand{\arraystretch}{1.2}
\begin{tabular}{@{}lcc@{}}
\toprule 
 & $\ell_\infty$ bound & $\ell_2$ bound \\ 
\midrule 
Full matrix & $\dfrac{\sqrt{2} d B^2}{n}$ & $\dfrac{\sqrt{2} C^2}{n}$ \\[12pt]
Single entry & $\dfrac{2 B^2}{n}$ & $\dfrac{C^2}{n}$ \\
\bottomrule
\end{tabular}
\end{table}

\subsection{Sensitivity of full matrix}

The sensitivity is
\begin{align*}
  \Delta_\Sigma&= \sup_{X \sim X'} \| \Sigma(X') - \Sigma(X) \|_F 
  = \sup_{v, u \in \mathcal X} \bigg\| \frac{1}{n} (vv^T - uu^T) \bigg\|_F
\end{align*}
where $X'$ is obtained from $X$ by adding record $v$ and removing $u$.

Expanding the Frobenius norm and rearranging, we have
\begin{align*}
    \Delta_\Sigma
    &= \frac{1}{n} \sup_{v, u \in \mathcal X} \sqrt{\sum_{i,j} (v_i v_j - u_i u_j)^2} \\
    &= \frac{1}{n} \sup_{v, u \in \mathcal X} \sqrt{\sum_{i,j} \left( v_i^2 v_j^2 + u_i^2 u_j^2 - 2\,v_i v_j\, u_i u_j \right)}\\
    &= \frac{1}{n} \sup_{v, u \in \mathcal X} \sqrt{ \Big( \sum_i v_i^2 \Big)^2 + \Big( \sum_i u_i^2 \Big)^2 - 2 \Big( \sum_i v_i u_i \Big)^2 }\\[1ex]
    &= \frac{1}{n} \sup_{v, u \in \mathcal X} \sqrt{ \|v\|_2^4 + \|u\|_2^4 - 2\,(v\cdot u)^2 } \\
    &\leq \frac{1}{n} \sup_{v, u \in \mathcal X} \sqrt{ \|v\|_2^4 + \|u\|_2^4 }.
\end{align*}

\begin{itemize}[leftmargin=*]
  \item \textbf{$\ell_2$-ball assumption.} If we assume $\|v\|_2,\|u\|_2 \le C$, then $\sqrt{ \|v\|_2^4 + \|u\|_2^4 } \le \sqrt{C^4 + C^4} = \sqrt{2}\,C^2$, so $\Delta_\Sigma \leq \frac{\sqrt{2}\,C^2}{n}$. 
  Further, assuming $d \geq 2$ all inequalities above are tight and the bound is achieved when $v = C e_i$ and $u = C e_j$ for $i \neq j$, where $e_i, e_j$ are standard basic vectors. Thus
  \begin{align*}
     \Delta_\Sigma &= \frac{\sqrt{2}\,C^2}{n}.
  \end{align*}

  \item \textbf{$\ell_\infty$-ball assumption.} If we assume $\|v\|_\infty, \|u\|_\infty \leq B$ then $\|v\|_2, \|u\|_2 \leq C := \sqrt{d}\,B$, so by the case above we have
  \begin{align*}
    \Delta_\Sigma \leq \frac{\sqrt{2}\, d\, B^2}{n}.
  \end{align*}
  This bound is tight when $d$ is even and $v = (B, B, \ldots, B, B)$ and $u = (B, B, \ldots, -B, -B)$ with $d/2$ entries of $B$ and $d/2$ entries of $-B$, so that $\|u|_2 = \|v\|_2 = \sqrt{d} B$ and $v \cdot u = 0$, making all inequalities tight.
\end{itemize}

\subsection{Sensitivity of a single entry}

The sensitivity of a single entry $\Sigma_{jk}(X)$ for $j \neq k$ can be written as
\begin{align*}
    \Delta_{\Sigma_{jk}} = \sup_{X \sim X'} \big|\Sigma_{jk}(X') - \Sigma_{jk} (X) \big| = 
    \frac{1}{n} \sup_{v,u \in \mathcal X} \big|v_j v_k - u_j u_k \big|
\end{align*}
where again $X'$ is obtained from $X$ by adding record $v$ and removing record $u$.

\begin{itemize}[leftmargin=*]
\item \textbf{$\ell_2$-ball assumption.} If we assume $\|v\|_2 \leq C$ and $\|u\|_2 \leq C$, then for $j \neq k$,
\[
|v_j v_k - u_j u_k|
\leq |v_j v_k| + |u_j u_k|
\leq \frac{C^2}{2} + \frac{C^2}{2}
= C^2,
\]
where the bound is tight, for example, when 
$v_j=v_k=u_j=-u_k=C/\sqrt{2}$. For diagonal entries $j=k$, we instead have
\[
|v_j^2-u_j^2| \leq C^2,
\]
since $v_j^2,u_j^2 \in [0,C^2]$. Therefore, for all entries,
\[
\Delta_{\Sigma_{jk}} = \frac{C^2}{n}.
\]

\item \textbf{$\ell_\infty$-ball assumption.} 
If we assume $\|v\|_\infty \leq B$ and $\|u\|_\infty \leq B$, then 
\[
|v_j v_k - u_j u_k| \leq |v_j v_k| + |u_j u_k| \leq B^2 + B^2 = 2B^2,
\]
and the inequalities are tight when $v_j = v_k = u_j = -u_k = B$, so
\[
 \Delta_{\Sigma_{jk}} = \frac{2B^2}{n}.
\]
\end{itemize}





\section{Proofs and analysis for the reconstruction problem}

\subsection{Properties of the maximum-entropy reconstruction}
\label{app:properties}

\begin{proof}[Proof of Proposition~\ref{prop:sparse-inverse}]
Because the objective $\mathcal L(W)$ depends only $W_{jk}$ for $(j, k) \in S$, we can reformulate it as $\mathcal L(W) = f(\pi_S(W))$ where $\pi_S(W) = (W_{jk})_{(j,k) \in S}$ extracts the entries indexed by $S$ and $f(z) = \sum_{(j,k) \in S} \frac{(z_{jk} - y_{jk})^2}{2 \tau_{jk}^2}$ for a vector $z$ indexed by $S$. 
Then, by construction,
\[
\hat W \in \argmin_{W \succeq 0} \mathcal L(W) \quad \iff \pi_S(\hat W) \in \argmin_{z \in Z} f(z), \quad Z := \{\pi_S(W): W \succeq 0\}.
\]
The problem on the right-hand side has a unique minimizer $a$ because $f(z)$ is strictly convex and $Z$ is convex.
Therefore $\hat W \in \argmin_{W \succeq 0} \mathcal L(W)$ if and only if $\pi_S(\hat W) = a$. 
In other words, the set of minimizers of $\mathcal L(W)$ is exactly the set of PSD matrices whose entries in $S$ are equal to $a$.

The outer problem now becomes exactly Dempster's covariance selection problem: find the PSD matrix $\Sigma$ with maximum entropy such that $\Sigma_{jk} = a_{jk}$ for all $(j,k) \in S$.
The result follows directly from Dempster's \emph{existence and uniqueness} property; the required condition that there is a positive definite matrix $\Sigma$ with $\pi_S(\Sigma) = a$ is satisfied by our assumption that $\argmin_W \mathcal L(W)$ contains a positive definite matrix and our characterization of this set.
\end{proof}

\subsection{Discussion of ill posedness}
\label{app:ill-posedness}
In most literature on covariance selection there is little discussion of the problem being ill posed. 
As Dempster showed, as long as there is \emph{some} positive definite matrix matching the specified entries, the covariance selection problem has a unique solution. 
Since the original empirical covariance matrix matches the specified entries, as long as this matrix is positive definite, the problem is well posed.

The situation is significantly different with noisy observations. 
To understand why, it is useful to consider the alternate characterization of the minimizers of $\mathcal L(W)$ in the proof of Proposition~\ref{prop:sparse-inverse}.
The feasible set $Z = \{\pi_S(W): W \succeq 0\}$ is the set of all "PSD-completable" vectors $z$, i.e., those such that there is a PSD matrix $W$ with $\pi_S(W) = z$. 
This is a slice of the PSD cone and is also a convex cone. 
(See the closely related discussion of PSD-completable matrices in~\cite{sun2015decomposition}.) From the form of the objective $f$, we see that the operation $a = \argmin_{z \in Z} f(z)$ computes a coordinate-scaled $\ell_2$ projection of $y$ onto $Z$, where $y = \pi_S(\Sigma) + \text{noise}$.
If the noise scale is large enough, it is very likely that $y$ falls outside $Z$, in which case the projection $a$ will lie on the boundary of $Z$.
A point $a$ on the boundary is such that \emph{every} $W$ with $\pi_S(W) = a$ is singular, meaning that every minimizer of $\mathcal L(W)$ is singular. 
Thus, the reconstruction problem is ill posed: there is no unique solution, and all minimizers have entropy $-\infty$.

As a simple example, consider the case when $S$ contains only indices of diagonal entries. Then $a_{jj} = y_{jj} \vee 0$ for all $j$. If any $y_{jj}$ is negative we have $a_{jj} = 0$, and every PSD matrix with this diagonal is singular.

These considerations are largely resolved by our convention in Section~\ref{sec:optimization} of defining the solution as the limit of a sequence of entropy-regularized solutions.
This gives the unique maximum-entropy solution when the problem is well posed, and otherwise defines the desired solution in a princpled way through the sequence of regularized problems.

\subsection{Inverse-variance weighting for repeated measurements}
\label{app:ivw}
\begin{proof}[Proof of Proposition~\ref{prop:ivw}]
The result follow from manipulating the obective to group terms by entry and then applying Lemma~\ref{lem:ivw} below.
\begin{align*}
\mathcal L(W) 
&= \sum_{t=1}^T \sigma_t^{-2} (W_{j_t k_t} - z_t )^2 \\
&= \sum_{(j, k)} \sum_{t \in \mathbb T_{jk}} \sigma_t^{-2} (W_{jk} - z_t)^2 \\
&= \sum_{(j, k)} \lambda_{jk} (W_{jk} - y_{jk})^2 + C
\end{align*}
where $\lambda_{jk} = \sum_{t \in \mathbb T_{jk} } \sigma_{t}^{-2}$ and $y_{jk} = \frac{1}{\lambda_{jk}} \sum_{t \in \mathbb T_{jk}} \sigma_t^{-2} z_t$.
The final line uses Lemma~\ref{lem:ivw}.
\end{proof}

\begin{lemma}
  \label{lem:ivw}
  $\sum_{q} w_q (x - a_q)^2 = \big(\sum_q w_q\big) (x - \bar a)^2  + C$ where $\bar a := \frac{\sum_q w_q a_q}{\sum_{q} w_q}$ and $C$ is constant with respect to $x$. 
\end{lemma}

\begin{proof}
Expanding the left-hand side:
\begin{align*}
\sum_{q} w_q (x - a_q)^2 
&= \sum_{q} w_q (x^2 - 2xa_q + a_q^2) \\
&= x^2 \sum_q w_q - 2x \sum_q w_q a_q + \sum_q w_q a_q^2 \\
&= \left(\sum_q w_q\right) \left(x^2 - 2x\frac{\sum_q w_q a_q}{\sum_q w_q} + \left(\frac{\sum_q w_q a_q}{\sum_q w_q}\right)^2\right) \\
&\quad + \sum_q w_q a_q^2 - \frac{(\sum_q w_q a_q)^2}{\sum_q w_q} \\
&= \left(\sum_q w_q\right) (x - \bar{a})^2 + C
\end{align*}
where the constant $C = \sum_q w_q a_q^2 - \frac{(\sum_q w_q a_q)^2}{\sum_q w_q} = \sum_q w_q(a_q - \bar{a})^2$ is independent of $x$.
\end{proof}

\section{Solver}\label{app:ipm}
The \textsc{IPM-Cholesky} solver is outline in Algorithm~\ref{alg:ipm}.

\begin{algorithm}[H]
\caption{\textsc{IPM-Cholesky}}\label{alg:ipm}
\begin{algorithmic}[1]
\Require Observations $\mathrm{Obs}$, variances $\mathrm{Var}$, initial barrier $\mu_0$, reduction factor $c_\mu$, minimum barrier $\mu_{\min}$, initial $L_0$ and $\mu_0$
\Ensure Maximum-entropy covariance estimate $\hat{\Sigma}$
\State $\mu \gets \mu_0$

\While{$\mu > \mu_{\min}$}
    \State Set tolerance $\text{tol}(\mu)$ \Comment{Coarser early, finer near $\mu_{\min}$}
    \State $L \gets \textsc{L-BFGS-B}\left(\varphi_\mu,\, L\right)$ with tolerance $\text{tol}(\mu)$
    \State $\mu \gets c_\mu \cdot \mu$
\EndWhile

\State \Return $\hat{\Sigma} \gets LL^\top$
\end{algorithmic}
\end{algorithm}


\section{Datasets}\label{app:datasets}
We evaluate \PrivateGGM{} on fourteen real-world tabular numerical datasets spanning a wide
range of dimensions and correlation structures:

\begin{itemize}
    \item \texttt{Adult} ($d=6$)~\citep{adult_2}
    \item \texttt{SeoulBike} ($d=10$)~\citep{Seoul_bike_sharing_demand_560}
    \item \texttt{LifeExpectancy} ($d=15$)~\citep{kumarrajarshi2018lifeexpectancy}
    \item \texttt{BreastCancerWisconsin} ($d=32$)~\citep{breast_cancer_wisconsin_17}
    \item \texttt{CommunitiesAndCrime} ($d=102$)~\citep{communities_and_crime_183}
    \item From the TALENT benchmark~\citep{talent2025}: \texttt{Parkinsons} ($d=16$), \texttt{thyroid\_ann} ($d=21$), \texttt{ibm} ($d=23$), \texttt{pol\_reg} ($d=26$), \texttt{bank32nh} ($d=32$), \texttt{ailerons} ($d=33$), \texttt{spambase} ($d=57$),
    \texttt{indian\_pines} ($d=187$),
    \texttt{madeline} ($d=260$)
\end{itemize}

\begin{table}[H]
\centering
\caption{Dataset properties. $\bar{|r|}$: mean absolute pairwise correlation. $\max |r|$: maximum absolute pairwise correlation.
$e_{\mathrm{off}} = \|\Sigma_{\mathrm{off}}\|_F^2 / \|\Sigma\|_F^2$: fraction of covariance energy contained in the off-diagonal entries.
$\mathrm{rank}_{\mathrm{eff}}$: effective rank 
$\exp\!\left(-\sum_i p_i \log p_i\right)$ of $\Sigma$,
where $p_i = \nu_i / \sum_j \nu_j$ and $\{\nu_i\}$ are the eigenvalues of $\Sigma$.}
\label{tab:dataset-structure}
\renewcommand{\arraystretch}{1.2}
\begin{tabular}{@{}lrrrrrr@{}}
\toprule
Dataset & $d$ & $n$ & $\bar{|r|}$ & $\max|r|$ & $e_{\mathrm{off}}$ & $\text{rank}_{\mathrm{eff}}$ \\
\midrule
\texttt{adult}                               & 6 & 48{,}842 & 0.059 & 0.144 & 0.018 & 4.6 \\
\texttt{SeoulBike}                       & 10 & 8{,}760 & 0.203 & 0.913 & 0.340 & 5.3 \\
\texttt{LifeExpectancy}                      & 15 & 2{,}864 & 0.398 & 0.954 & 0.656 & 5.1 \\
\texttt{parkinsons}                          & 17 & 5{,}875 & 0.451 & 0.993 & 0.392 & 6.3 \\
\texttt{thyroid\_ann}                        & 21 & 3{,}772 & 0.052 & 0.772 & 0.026 & 10.1 \\
\texttt{ibm}                             & 23 & 1{,}470 & 0.081 & 0.951 & 0.230 & 17.6 \\
\texttt{pol\_reg}                            & 27 & 15{,}000 & 0.103 & 0.790 & 0.165 & 9.9 \\
\texttt{BreastCancerWisconsin}               & 32 & 569 & 0.383 & 1.000 & 0.739 & 5.5 \\
\texttt{bank32nh}                            & 33 & 8{,}192 & 0.015 & 0.523 & 0.071 & 26.4 \\
\texttt{ailerons}                            & 34 & 13{,}750 & 0.152 & 0.999 & 0.746 & 7.3 \\
\texttt{spambase}                            & 58 & 4{,}601 & 0.064 & 0.996 & 0.069 & 8.5 \\
\texttt{communities\_crime}                  & 102 & 1{,}994 & 0.245 & 0.996 & 0.895 & 17.2 \\
\texttt{indian\_pines} & 187 & 9144 & 0.555 & 0.999 & 0.859 & 2.5 \\
\texttt{madeline}                            & 260 & 3{,}140 & 0.016 & 0.991 & 0.248 & 201.2 \\
\bottomrule
\end{tabular}
\end{table}

\newpage
\section{Complete results}\label{app:plot-all}
In this Section, we show results for add datasets, and include comparison to an independent baseline that spends all privacy budget to release the diagonal but ignores off-diagonal correlations. This baseline tends to perform well only in cases where the covariance is nearly diagonal, without dense off-diagonal structure.

\begin{figure}[ht]
\centering
\includegraphics[width=0.85\linewidth]{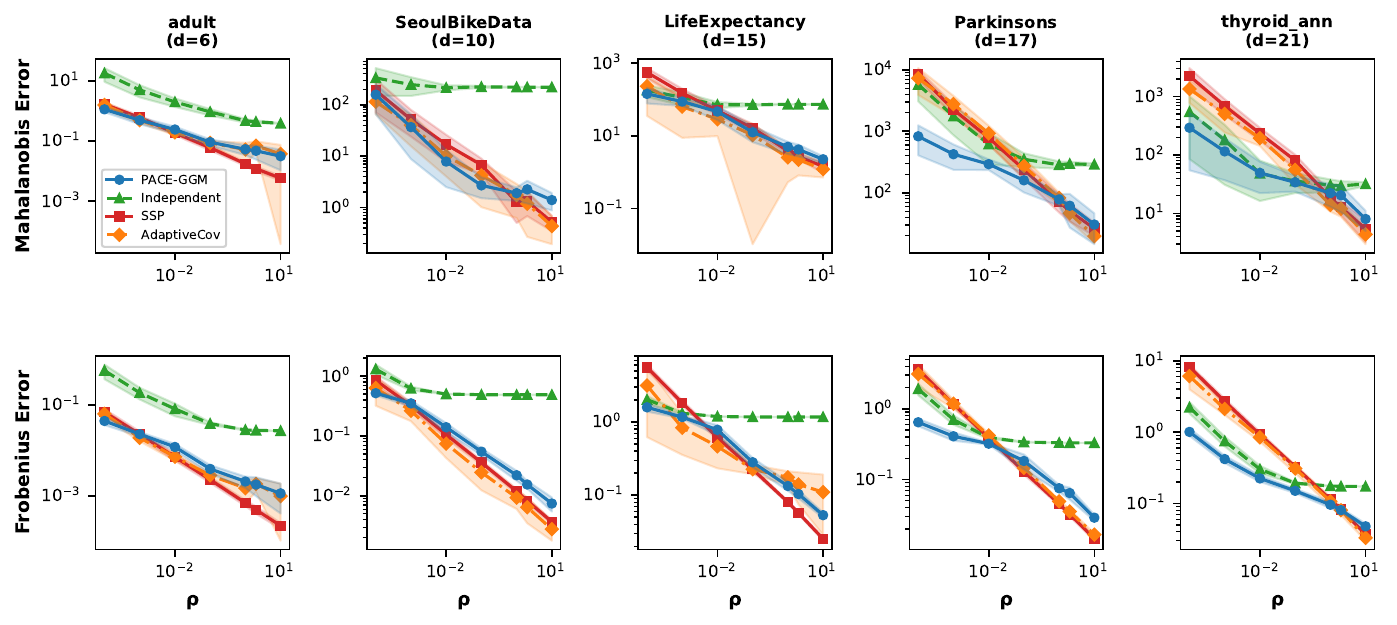}
\includegraphics[width=0.85\linewidth]{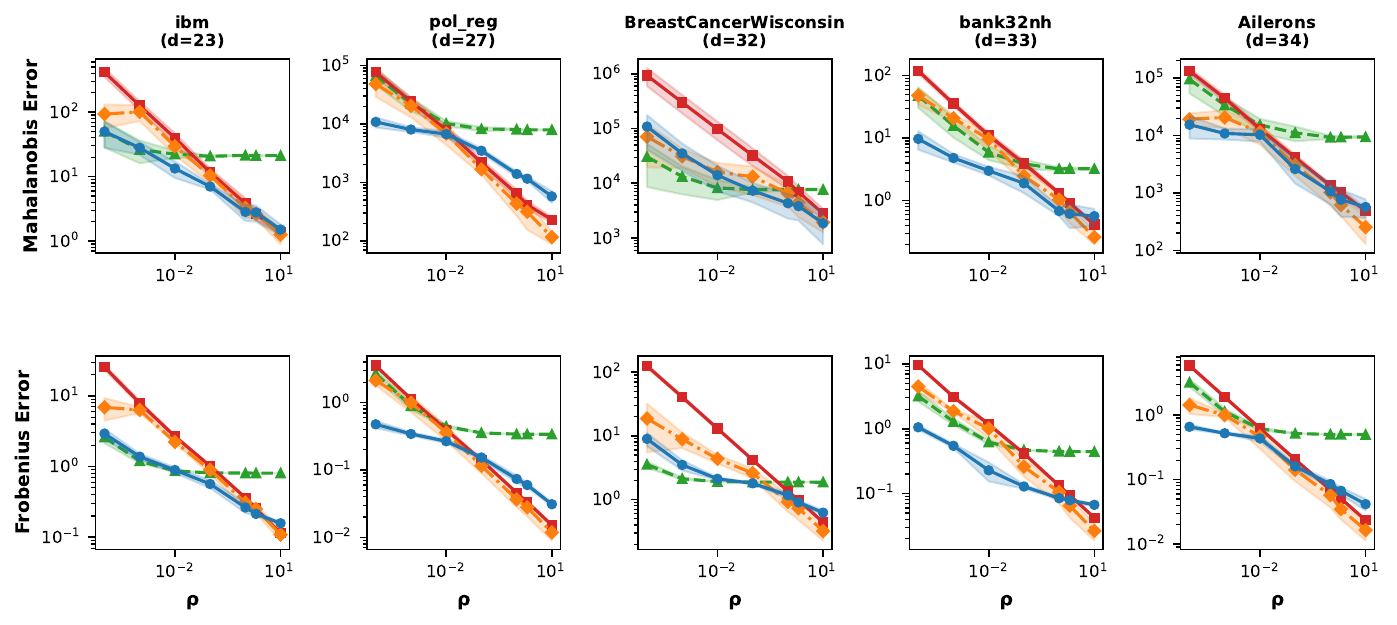}
\includegraphics[width=0.7\linewidth]{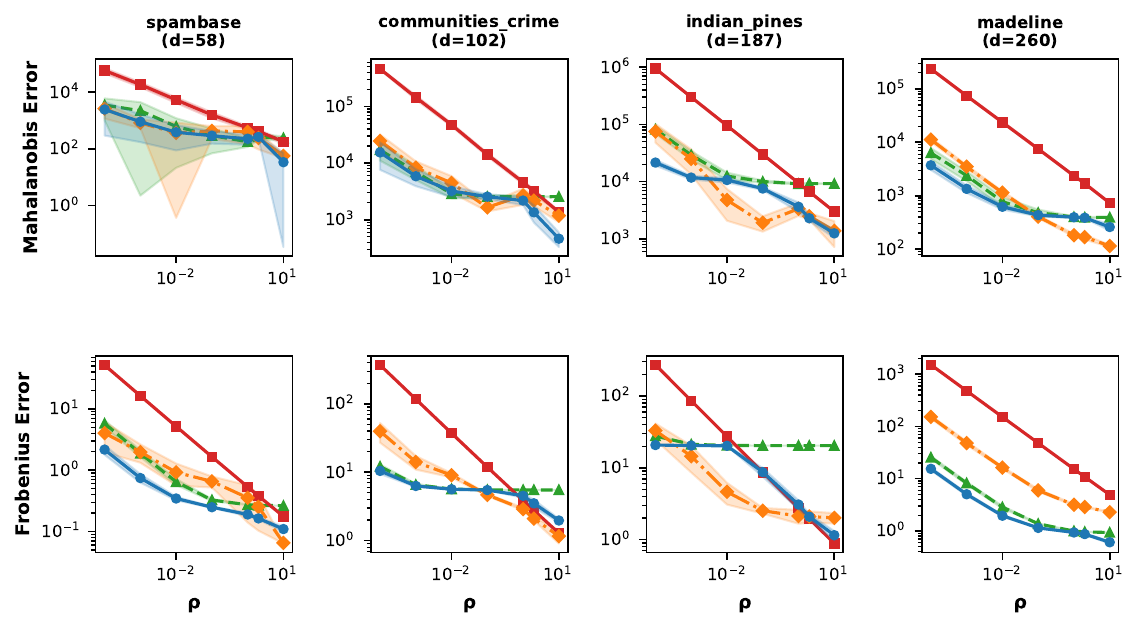}
\caption{Mahalanobis error (top row) and Frobenius error (bottom row) for all datasets. \PrivateGGM{} and baselines across privacy budgets $\rho$, with $\alpha=0.3$, $\beta=0.5$. Datasets are ordered by dimensionality $d$ across three groups.}
\label{fig:plot-all}
\end{figure}

\subsection{Number of measurements}
Table~\ref{tab:allmeasurements} shows the average number of measured entries for all datasets and $\rho$ values. 

\begin{table}[ht]
\centering
\resizebox{\linewidth}{!}{%
\begin{tabular}{llrrrrrrrr}
\toprule
Dataset & $d$ & $\frac{d(d+1)}{2}$ & $\rho=10^{-4}$ & $\rho=10^{-3}$ & $\rho=10^{-2}$ & $\rho=10^{-1}$ & $\rho=1$ & $\rho=2$ & $\rho=10$ \\
\midrule
adult                 & 6   & 21     & $6+2.6$   & $6+2.8$   & $6+3.4$   & $6+7.4$    & $6+11.6$   & $6+12.3$   & $6+12.2$   \\
SeoulBike             & 10  & 55     & $10+3.8$  & $10+9.7$  & $10+16.0$ & $10+21.8$  & $10+28.8$  & $10+30.3$  & $10+37.1$  \\
LifeExpectancy        & 15  & 120    & $15+4.9$  & $15+6.6$  & $15+13.5$ & $15+27.4$  & $15+43.7$  & $15+45.8$  & $15+61.4$  \\
Parkinsons            & 16  & 136    & $16+4.7$  & $16+4.4$  & $16+8.4$  & $16+19.6$  & $16+35.4$  & $16+40.3$  & $16+55.9$  \\
thyroid\_ann          & 21  & 231    & $21+5.0$  & $21+5.1$  & $21+6.2$  & $21+8.6$   & $21+19.0$  & $21+24.6$  & $21+42.5$  \\
ibm                   & 23  & 276    & $23+5.4$  & $23+5.8$  & $23+7.1$  & $23+11.6$  & $23+22.2$  & $23+28.4$  & $23+38.1$  \\
pol\_reg              & 26  & 351    & $26+5.4$  & $26+7.0$  & $26+10.3$ & $26+27.9$  & $26+71.6$  & $26+88.3$  & $26+120.6$ \\
BreastCancerWisconsin & 32  & 528    & $32+5.2$  & $32+5.6$  & $32+6.1$  & $32+9.4$   & $32+23.9$  & $32+27.5$  & $32+49.0$  \\
bank32nh              & 32  & 528    & $32+6.1$  & $32+8.1$  & $32+10.9$ & $32+12.1$  & $32+18.4$  & $32+25.5$  & $32+49.7$  \\
ailerons              & 33  & 561    & $33+5.3$  & $33+8.3$  & $33+13.4$ & $33+26.1$  & $33+53.4$  & $33+58.6$  & $33+95.8$  \\
spambase              & 57  & 1653   & $57+7.8$  & $57+7.8$  & $57+8.5$  & $57+8.1$   & $57+10.9$  & $57+11.5$  & $57+16.9$  \\
communities\_crime    & 102 & 5253   & $102+7.7$ & $102+8.8$ & $102+11.0$& $102+23.3$ & $102+97.5$ & $102+130.2$& $102+255.3$\\
indian\_pines & 187 & 17578 & $187+10.1$ & $187+18.8$ & $187+47.9$ & $187+199.2$ & $187+372.5$ & $187+410.2$ & $187+518.0$ \\
madeline              & 260 & 33930  & $260+9.1$ & $260+9.5$ & $260+9.9$ & $260+11.3$ & $260+15.0$ & $260+19.7$ & $260+39.2$ \\
\bottomrule
\end{tabular}%
}
\caption{Mean number of measurements per trial (diagonal $+$ off-diagonal). The column $d(d+1)/2$ gives the total number of unique entries in the covariance matrix (the maximum possible number of measurements).}
\label{tab:allmeasurements}
\end{table}

\section{Solver Findings}
\label{app:solver}

At each round, the privacy budget is allocated over a subset of variable pairs, defining a \emph{measurement graph} with one node per variable and one edge per observed covariance entry $\Sigma_{jk}$. The PSD reconstruction is then solved over this graph.

We study two aspects of the solver: (i) the choice of optimization method, and (ii) the use of sparsity via connected components.

\paragraph{IPM vs.\ PGD.}
We compare our maximum-entropy formulation, solved via an interior-point method (IPM), to projected gradient descent (PGD) applied directly to the PSD projection problem. PGD requires initialization for unobserved entries of the precision matrix; we consider two simple strategies: filling these entries with $0$ or $1$.

Figure~\ref{fig:pgd-ipm} shows that PGD is highly sensitive to initialization. While zero-initialization yields performance comparable to IPM, one-initialization overall leads to increased error and runtime, particularly in higher dimensions. In contrast, IPM consistently converges to a stable solution without requiring initialization tuning. This supports our use of the maximum-entropy objective as a principled and robust reconstruction method.

\paragraph{Connected components.}
When the measurement graph is disconnected, the PSD reconstruction decomposes exactly across connected components. If the graph consists of components of sizes $\{d_i\}$, the computational cost reduces from $O(d^3)$ to $O\!\left(\sum_i d_i^3\right)$, without approximation.

Figure~\ref{fig:ipm-cc} shows that leveraging connected components yields identical estimation error while consistently reducing runtime. 

Based on these findings, we use IPM with connected components as the default solver in all experiments.

\begin{figure}[ht]
    \centering
    \includegraphics[width=\linewidth]{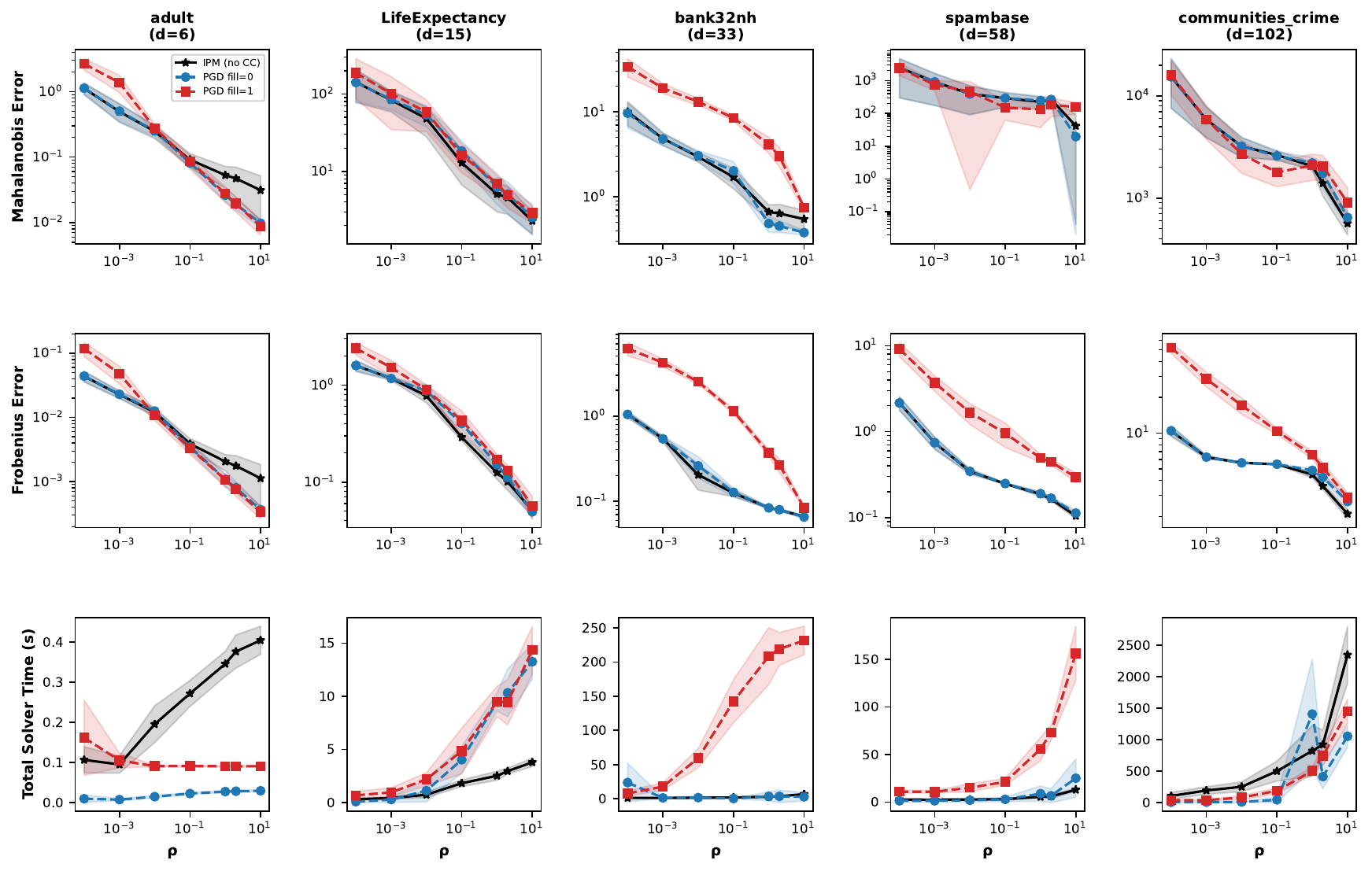}
    \caption{
        Solver: IPM vs.\ PGD across five datasets.
        Each column is a dataset ordered by dimension ($d=6$ to $d=102$).
        \textit{Top two rows}: Mahalanobis and Frobenius covariance estimation error (mean~$\pm$~std over 10 trials) as a function of the privacy budget~$\rho$.
        PGD's performance is dependent on initialization, while IPM's maximum-entropy approach converges reliably without need for initialization tuning.
        \textit{Bottom row}: total solver time per run (seconds).
        IPM is consistently fast and stable across all budgets and dimensions, and we use it as the default solver throughout.
    }
    \label{fig:pgd-ipm}
\end{figure}

\begin{figure}[h!]
    \centering
    \includegraphics[width=\linewidth]{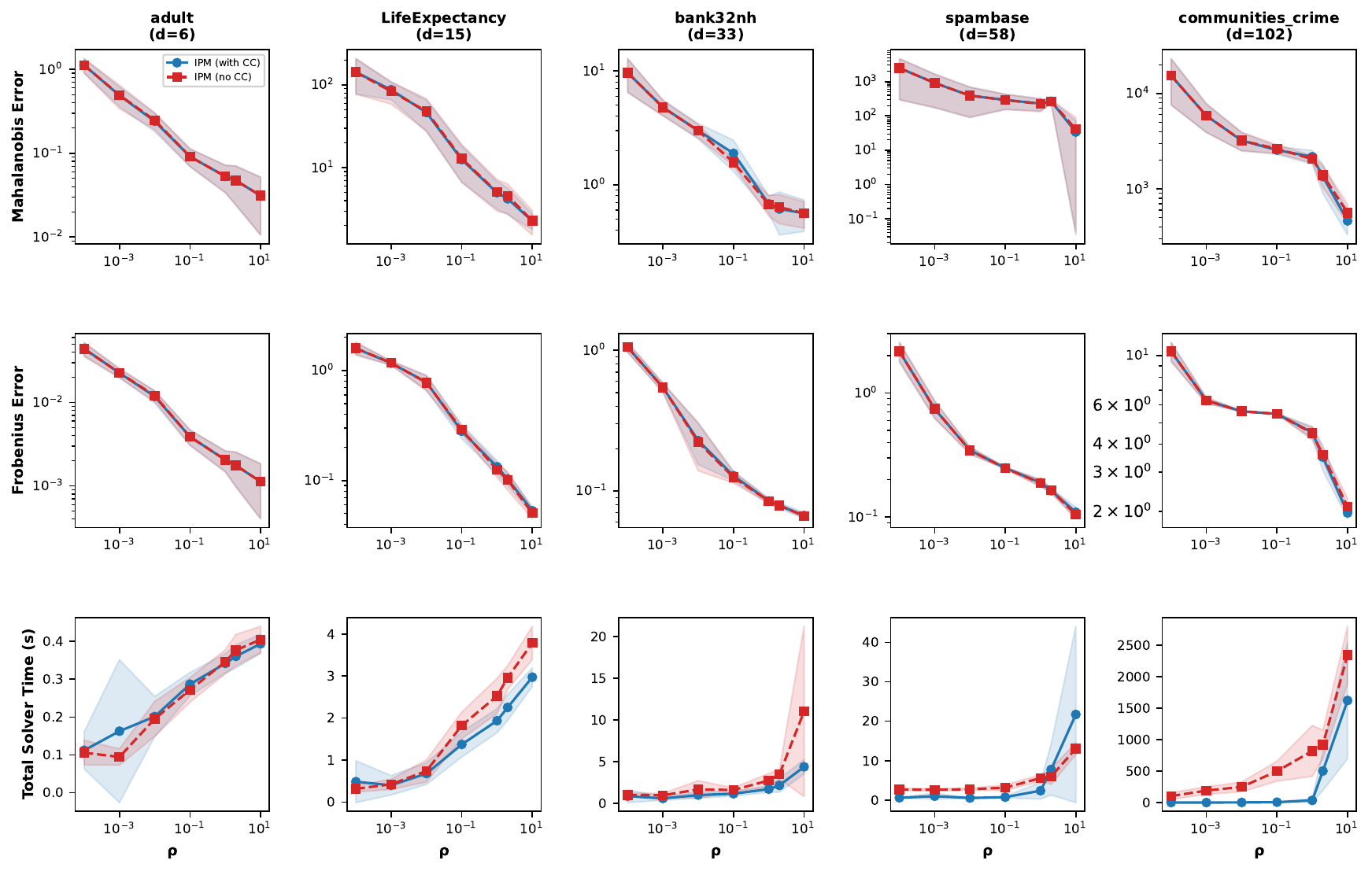}
    \caption{
        Solver: IPM without connected components (no CC) vs.\ IPM with connected components (CC) across five datasets.
        Each column is a dataset ordered by dimension ($d=6$ to $d=102$).
        \textit{Top two rows}: Mahalanobis and Frobenius covariance estimation error (mean~$\pm$~std over 10 trials) as a function of the privacy budget~$\rho$.
        The two approaches yield the same performance, with the connected components feature allowing runtime speedup by leveraging sparsity of the measurement graph.
    }
    \label{fig:ipm-cc}
\end{figure}

\newpage

\section{Regression}\label{app:regression}

As an application, we evaluate \PrivateGGM{} for linear regression. We compare against \textsc{AdaSSP}~\citep{wang2018revisiting}, a widely used baseline that applies \SSP{} to privately estimate the sufficient statistics $X^\top X$ and $X^\top y$, and allocates part of the privacy budget to estimating the minimum eigenvalue for regularization.
To ensure a fair comparison, we follow the same budget allocation strategy, reserving one third of the budget for regularization. However, instead of estimating sufficient statistics via \SSP{}, we derive them from the covariance matrix produced by \PrivateGGM{}.
A 80/20 train-test split is adopted for evaluation.
Figure~\ref{fig:regression} reports results on eight regression datasets. \PrivateGGM{} achieves competitive performance with \textsc{AdaSSP}, in most cases slightly outperforming it, while remaining a more general-purpose method, as it targets covariance estimation rather than a specific downstream task.

\begin{figure}[h!]
    \centering
    \includegraphics[width=0.95\linewidth]{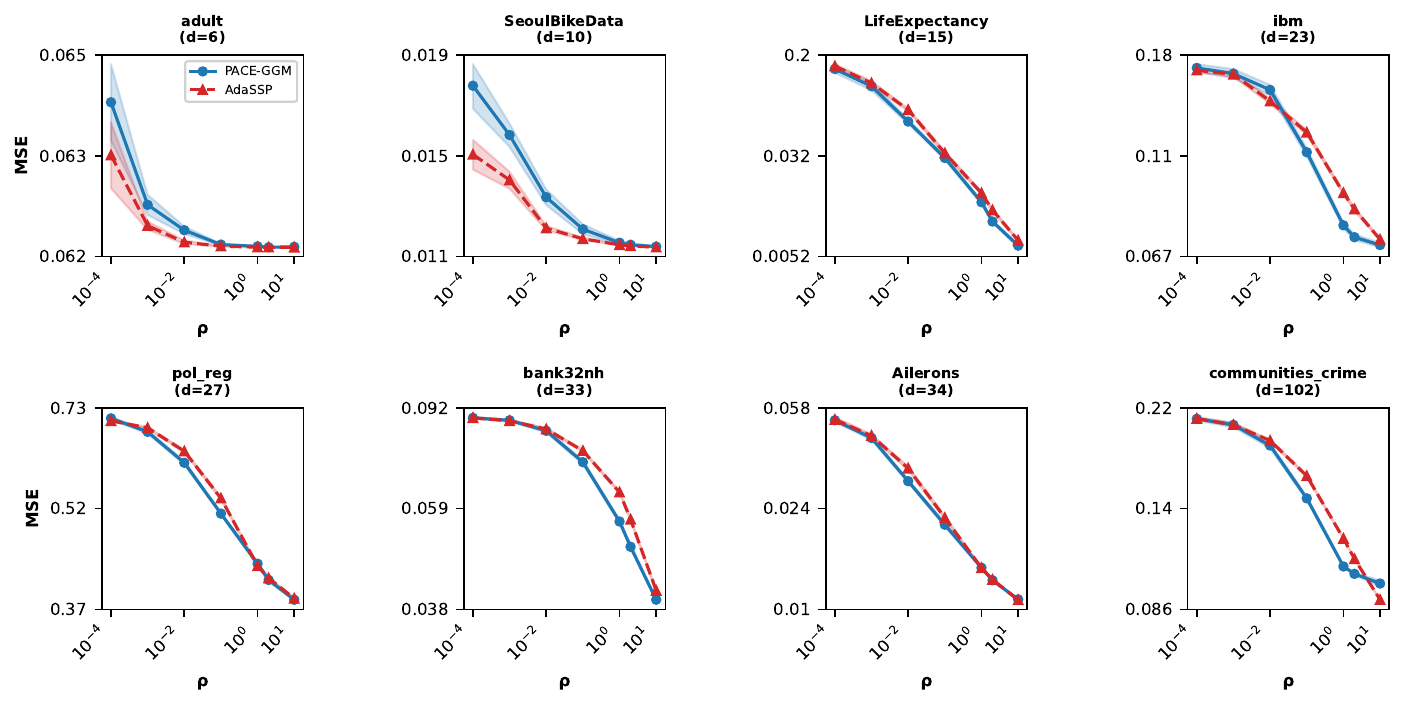}
    \caption{Mean square error (MSE) of \PrivateGGM{} and \textsc{AdaSSP}, across $\rho$ values ranging from $10^{-4}$ to $10$. Standard error bars are computed across 10 independent trials.}
    \label{fig:regression}
\end{figure}



\section{Runtime}
Table~\ref{tab:trial-times} reports mean total runtime per trial across all datasets and privacy budgets. 
For low-to-medium dimensional datasets ($d \leq 57$), runtimes are well under 3 seconds even at $\rho = 10$, confirming that \PrivateGGM{} is practical across the full budget range.
Runtime grows with both $d$ and $\rho$: larger $\rho$ means more budget to spend, so the algorithm runs more rounds and measures more entries before terminating.
This interaction between $d$ and $\rho$ is not monotone in $d$ alone ---
\texttt{madeline} ($d=260$) is paradoxically fast because its covariance is nearly diagonal (mean $|r| = 0.016$, effective rank $\text{rank}_{\mathrm{eff}} = 201$),
and the algorithm terminates after very few rounds.
\texttt{indian\_pines} ($d=187$), by contrast, is highly correlated ($\text{rank}_{\mathrm{eff}} = 2.6$), driving runtimes of thousands of seconds at high $\rho$ as many pairs remain informative across many rounds.
All experiments were run on a machine with an Apple M-series processor (10-core CPU, 10-core GPU, 16-core Neural Engine), 32 GB of unified memory, and 256 GB SSD storage.

\begin{table}[t]
\centering
\small
\caption{Mean total trial time (seconds) $\pm$ SE across 10 trials for selected $\rho$ values.}
\label{tab:trial-times}
\renewcommand{\arraystretch}{1.2}
\begin{tabular}{@{}lrcccc@{}}
\toprule
Dataset & $d$ & $\rho=10^{-4}$ & $\rho=10^{-2}$ & $\rho=1$ & $\rho=10$ \\
\midrule
\texttt{adult}                 & 6   & $0.05 \pm 0.01$ & $0.05 \pm 0.01$ & $0.13 \pm 0.01$ & $0.16 \pm 0.01$ \\
\texttt{SeoulBikeData}         & 10  & $0.08 \pm 0.01$ & $0.23 \pm 0.01$ & $0.41 \pm 0.02$ & $0.55 \pm 0.02$ \\
\texttt{LifeExpectancy}        & 15  & $0.12 \pm 0.02$ & $0.28 \pm 0.04$ & $0.95 \pm 0.06$ & $1.24 \pm 0.03$ \\
\texttt{Parkinsons}            & 16  & $0.12 \pm 0.01$ & $0.18 \pm 0.02$ & $0.68 \pm 0.05$ & $1.10 \pm 0.05$ \\
\texttt{thyroid\_ann}          & 21  & $0.14 \pm 0.02$ & $0.19 \pm 0.02$ & $0.35 \pm 0.04$ & $0.64 \pm 0.04$ \\
\texttt{ibm}                   & 23  & $0.18 \pm 0.02$ & $0.21 \pm 0.02$ & $0.64 \pm 0.07$ & $0.86 \pm 0.06$ \\
\texttt{pol\_reg}              & 26  & $0.16 \pm 0.02$ & $0.24 \pm 0.02$ & $1.90 \pm 0.13$ & $2.76 \pm 0.10$ \\
\texttt{BreastCancerWisconsin} & 32  & $0.19 \pm 0.02$ & $0.20 \pm 0.03$ & $0.90 \pm 0.17$ & $1.78 \pm 0.15$ \\
\texttt{bank32nh}              & 32  & $0.20 \pm 0.02$ & $0.38 \pm 0.03$ & $0.57 \pm 0.09$ & $1.52 \pm 0.17$ \\
\texttt{Ailerons}              & 33  & $0.18 \pm 0.02$ & $0.46 \pm 0.05$ & $1.61 \pm 0.12$ & $2.96 \pm 0.15$ \\
\texttt{spambase}              & 57  & $0.34 \pm 0.04$ & $0.37 \pm 0.04$ & $0.42 \pm 0.04$ & $0.65 \pm 0.09$ \\
\texttt{communities\_crime}    & 102 & $0.39 \pm 0.04$ & $0.63 \pm 0.06$ & $10.4 \pm 1.6$  & $47.8 \pm 2.8$  \\
\texttt{indian\_pines}         & 187 & $0.70 \pm 0.07$ & $8.01 \pm 2.22$ & $2714 \pm 369$  & $3117 \pm 697$             \\
\texttt{madeline}              & 260 & $0.83 \pm 0.11$ & $0.84 \pm 0.11$ & $1.28 \pm 0.10$ & $4.16 \pm 0.32$ \\
\bottomrule
\end{tabular}
\end{table}

\clearpage
\section{Hyperparameters $\alpha$ and $\beta$}\label{app:hyperparameters}

In this Section we study hyperparameters $\alpha \in (0,1)$, the fraction of $\rho$ reserved for diagonal initialization, and $\beta \in (0,1)$, the fraction of the per-round budget allocated to the exponential-mechanism selection step. We evaluate all nine combinations $(\alpha,\beta) \in \{0.1,0.3,0.5\}^2$ on five datasets (\texttt{Adult}, \texttt{LifeExpectancy}, \texttt{IBM}, \texttt{Bank32nh}, \texttt{Spambase}) across $\rho \in \{10^{-4},10^{-3},10^{-2}, 0.1, 1, 2, 10\}$ with five trials per setting. Results are in Tables~\ref{tab:grid_best_ab_frob}--\ref{tab:grid_best_ab_mah}.

\textbf{Effect of $\alpha$.} The optimal diagonal budget fraction exhibits interpretable dependence on $\rho$, consistent with the theoretical role of each phase. At low budgets ($\rho \leq 10^{-3}$), $\alpha=0.5$ wins in 8/10 cells: when measurements are highly noisy, a well-conditioned diagonal initialization prevents wasted budget on near-zero off-diagonal entries. At higher budgets ($\rho \geq 0.1$), $\alpha=0.3$ dominates for Mahalanobis error, reflecting that off-diagonal structure becomes resolvable and warrants more investment. This budget-regime dependence is exactly what one would expect from the algorithm's design. We set $\alpha=0.3$ as the default, as Mahalanobis error better captures accuracy for downstream inference.

\textbf{Effect of $\beta$.} The results confirm the core design intuition: accurate entry selection is critical. $\beta=0.5$ achieves the lowest Frobenius error in 25/35 dataset$\times\rho$ cells and 14/20 in the high-budget regime $\rho \in \{0.1,1,2,10\}$; $(\alpha{=}0.3,\,\beta{=}0.5)$ is the single most frequent Mahalanobis winner with 13/35 cells. Investing half the per-round budget in selection consistently pays off. We fix $\beta=0.5$ for all main experiments.

\textbf{Robustness to hyperparameter choice.} Crucially, performance is not sensitive to the exact choice of $(\alpha,\beta)$: all combinations in $(\alpha,\beta) \in \{0.1,0.3,0.5\}^2$ remain competitive across datasets and privacy budgets. Figure~\ref{fig:plot-hyperparameters} shows \PrivateGGM{}'s performance under all combinations of hyperparameters $\alpha$ and $\beta$, highlighting only a modest change in performance across datasets and $\rho$. This robustness is practically significant --- it means the algorithm does not require careful tuning or dedicated hyperparameter search privacy budget. The default $(\alpha{=}0.3,\,\beta{=}0.5)$ is a reliable choice across diverse settings.

\begin{table}[ht]
  \centering\scriptsize\setlength{\tabcolsep}{5pt}
  \resizebox{\linewidth}{!}{%
  \begin{tabular}{lccccccc}
  \toprule
  \textbf{Dataset} & $\rho{=}$$10^{-4}$ & $\rho{=}$$10^{-3}$ & $\rho{=}$$10^{-2}$ & $\rho{=}$$0.1$ & $\rho{=}$$1$ & $\rho{=}$$2$ & $\rho{=}$$10$ \\
  \midrule
  Adult & \textbf{$(0.5,\,0.1)$} & \textbf{$(0.3,\,0.5)$} & \textbf{$(0.1,\,0.5)$} & \textbf{$(0.1,\,0.3)$} & \textbf{$(0.3,\,0.1)$} & \textbf{$(0.1,\,0.1)$} & \textbf{$(0.5,\,0.3)$} \\
  LifeExp. & \textbf{$(0.5,\,0.1)$} & \textbf{$(0.3,\,0.3)$} & \textbf{$(0.1,\,0.5)$} & \textbf{$(0.3,\,0.5)$} & \textbf{$(0.3,\,0.5)$} & \textbf{$(0.1,\,0.5)$} & \textbf{$(0.1,\,0.5)$} \\
  IBM & \textbf{$(0.5,\,0.5)$} & \textbf{$(0.5,\,0.1)$} & \textbf{$(0.5,\,0.5)$} & \textbf{$(0.1,\,0.5)$} & \textbf{$(0.3,\,0.5)$} & \textbf{$(0.1,\,0.5)$} & \textbf{$(0.1,\,0.5)$} \\
  Bank32nh & \textbf{$(0.5,\,0.5)$} & \textbf{$(0.5,\,0.5)$} & \textbf{$(0.1,\,0.5)$} & \textbf{$(0.5,\,0.3)$} & \textbf{$(0.1,\,0.5)$} & \textbf{$(0.3,\,0.5)$} & \textbf{$(0.1,\,0.5)$} \\
  Spambase & \textbf{$(0.5,\,0.5)$} & \textbf{$(0.5,\,0.5)$} & \textbf{$(0.5,\,0.5)$} & \textbf{$(0.5,\,0.1)$} & \textbf{$(0.1,\,0.5)$} & \textbf{$(0.3,\,0.5)$} & \textbf{$(0.1,\,0.5)$} \\
  \bottomrule
  \end{tabular}}
  \caption{Frobenius error. Best $(\alpha,\beta)$ pair per dataset and $\rho$ (5 trials each). Top winning combinations by count: $(0.1, 0.5)$: 12; $(0.5, 0.5)$: 7; $(0.3, 0.5)$: 6.}
  \label{tab:grid_best_ab_frob}
\end{table}

\begin{table}[ht]
  \centering\scriptsize\setlength{\tabcolsep}{5pt}
  \resizebox{\linewidth}{!}{%
  \begin{tabular}{lccccccc}
  \toprule
  \textbf{Dataset} & $\rho{=}$$10^{-4}$ & $\rho{=}$$10^{-3}$ & $\rho{=}$$10^{-2}$ & $\rho{=}$$0.1$ & $\rho{=}$$1$ & $\rho{=}$$2$ & $\rho{=}$$10$ \\
  \midrule
  Adult & \textbf{$(0.5,\,0.3)$} & \textbf{$(0.5,\,0.5)$} & \textbf{$(0.3,\,0.5)$} & \textbf{$(0.1,\,0.3)$} & \textbf{$(0.5,\,0.5)$} & \textbf{$(0.5,\,0.1)$} & \textbf{$(0.5,\,0.3)$} \\
  LifeExp. & \textbf{$(0.5,\,0.5)$} & \textbf{$(0.5,\,0.1)$} & \textbf{$(0.1,\,0.5)$} & \textbf{$(0.1,\,0.5)$} & \textbf{$(0.3,\,0.5)$} & \textbf{$(0.1,\,0.1)$} & \textbf{$(0.3,\,0.5)$} \\
  IBM & \textbf{$(0.3,\,0.5)$} & \textbf{$(0.5,\,0.1)$} & \textbf{$(0.3,\,0.5)$} & \textbf{$(0.3,\,0.1)$} & \textbf{$(0.3,\,0.5)$} & \textbf{$(0.1,\,0.5)$} & \textbf{$(0.3,\,0.5)$} \\
  Bank32nh & \textbf{$(0.5,\,0.3)$} & \textbf{$(0.5,\,0.3)$} & \textbf{$(0.3,\,0.5)$} & \textbf{$(0.1,\,0.1)$} & \textbf{$(0.1,\,0.3)$} & \textbf{$(0.1,\,0.5)$} & \textbf{$(0.3,\,0.3)$} \\
  Spambase & \textbf{$(0.3,\,0.5)$} & \textbf{$(0.3,\,0.5)$} & \textbf{$(0.3,\,0.5)$} & \textbf{$(0.5,\,0.1)$} & \textbf{$(0.3,\,0.5)$} & \textbf{$(0.1,\,0.3)$} & \textbf{$(0.3,\,0.5)$} \\
  \bottomrule
  \end{tabular}}
  \caption{Mahalanobis error. Best $(\alpha,\beta)$ pair per dataset and $\rho$ (5 trials each). Top winning combinations by count: $(0.3, 0.5)$: 13; $(0.5, 0.3)$: 4; $(0.5, 0.1)$: 4.}
  \label{tab:grid_best_ab_mah}
\end{table}

\begin{figure}[ht]
    \centering
    \includegraphics[width=0.95\linewidth]{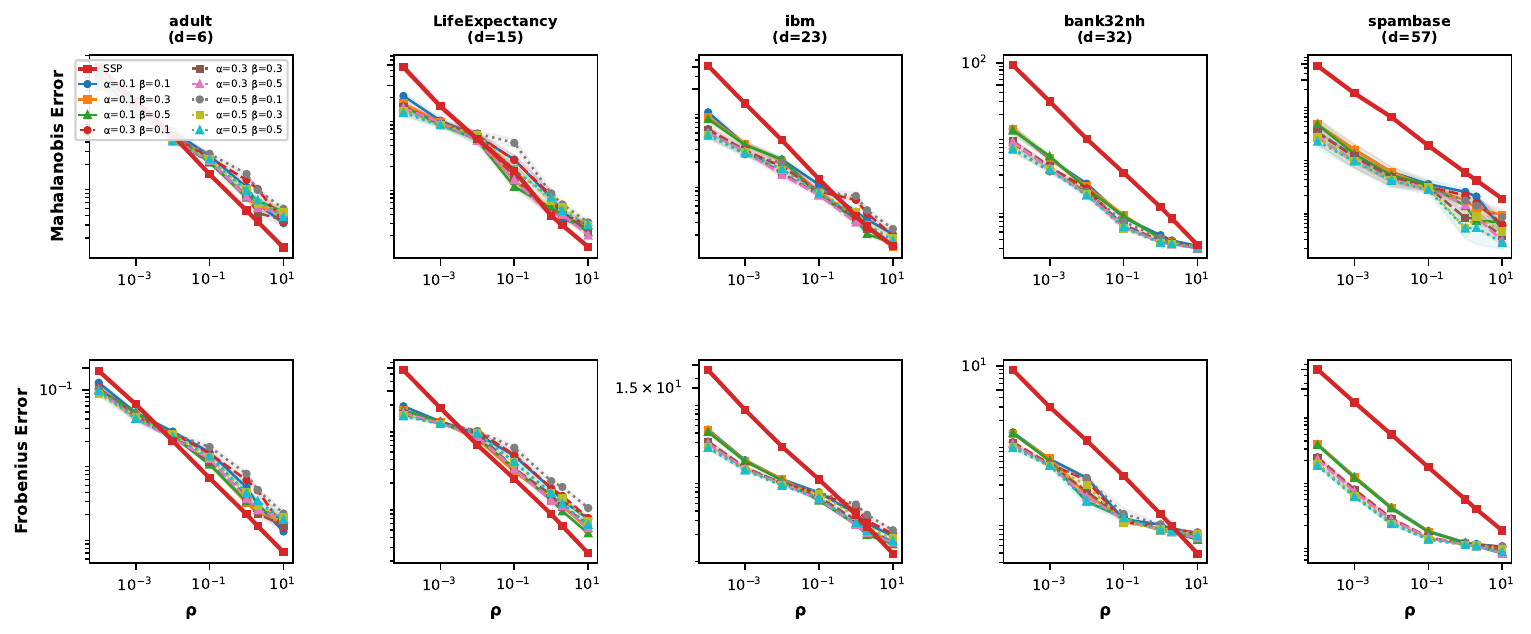}
    \caption{\PrivateGGM{} error under all nine combinations of $(\alpha,\beta) \in \{0.1,0.3,0.5\}^2$. \SSP{} is included as a baseline for reference. The performance is robust to the choice of $\alpha$ and $\beta$, with minimal performance changes.}
    \label{fig:plot-hyperparameters}
\end{figure}


\end{document}